\documentclass[preprint]{elsarticle}
\makeatletter
\def\ps@pprintTitle{%
 \let\@oddhead\@empty
 \let\@evenhead\@empty
 \def\@oddfoot{\centerline{\thepage}}%
 \let\@evenfoot\@oddfoot}
\makeatother
\usepackage[utf8]{inputenc}
\usepackage[T1]{fontenc}
\usepackage[margin=2.5cm]{geometry}

\usepackage{hyperref}
\usepackage{lineno}
\modulolinenumbers[5]
\usepackage[colorinlistoftodos,prependcaption,textsize=normalsize]{todonotes}
\usepackage{units}
\usepackage{subcaption}
\usepackage{booktabs}
\usepackage{makecell}
\usepackage{url}
\usepackage{amsmath,amssymb,amsfonts,amsthm}
\usepackage{mathtools}
\usepackage{bm, bbm}
\usepackage{nicefrac}
\usepackage{xcolor}
\usepackage{appendix}
\usepackage[symbol]{footmisc}

\usepackage[noabbrev]{cleveref}
\usepackage[normalem]{ulem}
\newtheorem{theorem}{Theorem}

\newtheorem{theo}{Theorem}

\newtheorem{lemma}[theo]{Lemma}
\newtheorem{propo}[theo]{Proposition}

\newtheorem{remark}[theo]{Remark}
\newtheorem{assumption}{Assumption}

\theoremstyle{definition}

\newcommand{\R}{{\ensuremath{\mathbb{R}}}}
\newcommand{\N}{{\ensuremath{\mathbb{N}}}}
\newcommand*{\Id}{\operatorname{Id}}
\newcommand*{\inte}{\operatorname{int}}
\newcommand*{\cl}{\operatorname{cl}}
\newcommand*{\Hess}{\operatorname{Hess}}

\newcommand\Hgen{\mathcal{H}^{\text{\normalfont gen}}}
\newcommand\Hdis{\mathcal{H}^{\text{\normalfont dis}}}
\newcommand\Hpot{\mathcal{H}^{\text{\normalfont pot}}}

\journal{Journal of \LaTeX\ Templates}

\begin{document}

\begin{frontmatter}

\title{Learning Brenier Potentials with Convex Generative Adversarial Neural Networks}

\author{C.\ Drygala${}^\dagger$, H.\ Gottschalk${}^\dagger$,
T.\ Kruse${}^\ast$, S.\ Martin${}^\dagger$ and A.\ Mütze${}^\ast$}
\address{${}^\ast$University of Wuppertal, School of Mathematics and Natural Sciences, IMACM \& IZMD\\
\{muetze, tkruse\}@uni-wuppertal.de\\
${}^\dagger$ Faculty II - Mathematics and Naural Sciences, Technical University Berlin \& Math+\\
\{drygala,gottschalk,martin\}@math.tu-berlin.de}
\begin{abstract}
Brenier proved that under certain conditions on a source and a target probability measure there exists a strictly convex function such that its gradient is a transport map from the source to the target distribution. This function is called the Brenier potential. Furthermore, detailed information on the H\"older regularity of the Brenier potential is available. In this work we develop the statistical learning theory of generative adversarial neural networks that learn the Brenier potential. As by the transformation of densities formula, the density of the generated measure depends on the second derivative of the Brenier potential, we develop the universal approximation theory of ReCU networks with cubic activation $\mathtt{ReCU}(x)=\max\{0,x\}^3$ that combines the favorable approximation properties of H\"older functions with a Lipschitz continuous density. In order to assure the convexity of such general networks, we introduce an adversarial training procedure for a potential function represented by the ReCU networks that combines the classical discriminator cross entropy loss with a penalty term that enforces (strict) convexity. We give a detailed decomposition of learning errors and show that for a suitable high penalty parameter all networks chosen in the adversarial min-max optimization problem are strictly convex. This is further exploited to prove the consistency of the learning procedure for (slowly) expanding network capacity. We also implement the described learning algorithm and apply it to a number of standard test cases from Gaussian mixture to image data as target distributions. As predicted in theory, we observe that the convexity loss becomes inactive during the training process and the potentials represented by the neural networks have learned convexity. 
\end{abstract}
\begin{keyword}
Generative adversarial networks $\bullet$ Brenier potential $\bullet$ convexity loss function. 
\end{keyword}
\end{frontmatter}
\nolinenumbers
\section{Introduction}

Generative learning models have become the fourth major family of machine learning algorithms besides supervised learning, unsupervised learning and reinforcement learning~\cite{bond2021deep,harshvardhan2020comprehensive}. In generative learning, the goal is to learn from a number of samples from a prescribed target distribution how to generate further samples that stem from a distribution that is as close to the target distribution as possible. As the generation usually involves sampling from some (pseudo) random number generator, the so-called source  distribution, generative learning can be understood as learning a map that transforms the source distribution to a  learned distribution that is as close to the target distribution as possible. 



A considerable number of generative algorithms have been proposed. Roughly, they can be subdivided in three categories: energy-based, map-based and flow-based. Restricted and deep Boltzmann machines and deep belief networks are classical energy based methods~\cite{montufar2011refinements,salakhutdinov2009deep}, where the target measure is represented as $d\mu^*(x)=\frac{1}{Z}e^{-V(x)}\,d\lambda(x)$, where $\lambda$ stands for the Lebesgue measure, see also~\cite{du2019implicit} for more recent models. The learning task is to learn the potential $V$ with some model $V(\cdot,\theta)$. However, the normalizing constant $Z(\theta)=\int e^{-V(x,\theta)} d\lambda(x)$ is notoriously intractable, which complicates training. As a result, computationally intensive Markov Chain Monte Carlo (MCMC) methods are needed for sampling.

The second family of algorithms is map-based and therefore directly relates to transport theory. The goal is to learn the transport map directly using some parametric family of maps, e.g.\ deep neural networks. Here we distinguish between explicit and implicit methods. Explicit methods -- like affine coupling flows~\cite{dinh2022density,papamakarios2021normalizing,teshima2020coupling} or LU-Nets~\cite{chan2023lu} -- rely on special layers like affine coupling blocks which make the density of the generated measure $\hat \mu$ easy to compute. This enables both, efficient generation and likelihood based training. However, the architecture limitations often prevent an efficient scaling to high dimension. Implicit methods, in contrast to this, do not require an explicit representation of the density, but generate samples during training. Generative adversarial networks (GANs)~\cite{vanilla_gan} use this strategy and are trained using the variational representation of $f$-divergences~\cite{nowozin2016f}, in particular of the Jensen-Shannon divergence. The maximization problem in the variational representation of the loss function then leads to a min-max kind of training which is not easily conducted numerically but, if done right, leads to convincing results in high dimensions, see e.g.~\cite{munoz2021temporal,wang2018esrgan}.

Finally, flow-based methods learn a continuous transformation of probability densities, where the terminal distribution represents the learned measure. NeuralODEs~\cite{chen2018neural} model this transformation differentially using the time-dependent flow associated with an ordinary differential equation (ODE). At the terminal time, the flow defines the transport map. In NeuralODE, the learned component is the vector field that governs the ODE dynamics. The training either is likelihood-based using Liouville's formula, or rely on a more recent flow matching technique~\cite{lipman2022flow,liu2022flow}. Diffusion models~\cite{sohl2015deep, welling2011bayesian,rombach2022high} replace the ODE by a stochastic differential equation and compute the flow of probability densities using Langevin's equation.


In this work, we explore GANs through the lens of optimal transport theory. In a GAN, the objective is to learn a generator $G$ that acts as a transport map, transforming samples from a source distribution into samples from a target (data) distribution. Inspired by a foundational result of Y.~Brenier~\cite{brenier1991polar}, we model the generator $G$ as the gradient of a scalar function, $G = \nabla \phi$, where $\phi$ is a convex potential. Brenier's theorem guarantees that, under mild regularity conditions, the optimal transport map between two distributions can be represented as the gradient of a convex function. Further developments in~\cite{paty2020regularity} provide insights into the Hölder regularity of such potentials, and it is well established that the Brenier potential is strictly convex~\cite{santambrogio2015optimal}.
Building on these results, we introduce \emph{Brenier GAN}, a framework that learns the potential function $\phi$ directly from data by parameterizing it with a deep neural network.


To build a statistical learning theory, the Brenier GAN framework offers significant advantages over the classical GAN formulation, in which the generator $G$ is directly modeled as a neural network. Indeed, the convexity constraint on the potential naturally addresses several challenges in GAN training. Notably, the gradient of a strictly convex function is injective. This enables one to provide estimates on the distance of the learned measures with respect to the target measure in terms of densities using the change of variables formula, provided the potential is twice differentiable. Importantly, this approach does not contradict the implicit nature of GANs, as density estimates are not explicitly computed during training or generation, they serve only for theoretical analysis. This set of assumptions also is in line with other works on the statistical learning theory of GANs~\cite{asatryan2020convenient,biau2020some,puchkin2024rates}.

Using  recent results on the universal approximation property (UAP) of deep neural networks~\cite{belomestny2023simultaneous} that give control both on function values and derivatives, we can not only approximate the Brenier potential, but also the generator and the target density. This however requires second order differentiable neural networks, which is not given for standard $\mathtt{ReLU}$-activated neural networks. Therefore we consider second order continuously differentiable neural networks with $\mathtt{ReCU}(x)=\max\{0,x\}^3$-activation. From a practical standpoint $\mathtt{ReQU}(x)=\max\{0,x\}^2$-activation would also be sufficient as it is still possible to back-propagate the first derivative of such nets. However, in this article we avoid technical problems in the context of only piecewise densities of the generated measures. We are however capable to reproduce the favorable universal approximation properties known for $\mathtt{ReQU}$ networks~\cite{belomestny2023simultaneous} for $\mathtt{ReCU}$ as well. 

Convexity is the second major challenge in the statistical learning theory of Brenier GAN. Because standard neural networks are not inherently convex, it is unclear whether a randomly chosen network preserves injectivity and satisfies the change of variables formula. Here we propose to resolve this problem by adding a penalty to the training loss which enforces convexity of the neural network. As a result, we obtain statistical estimates for the generalization error, providing uniform bounds on the difference between the empirical loss function and its expected value.

In this way we obtain an error decomposition for the Brenier GAN learning a convex potential. 
This potential does not necessarily have to be the Brenier potential, but Brenier's potential function and its properties play the central role in the control of the model error for the generator. As the Brenier GAN is trained using a discriminator network, also the error between the optimal discriminator based on densities and those discriminator functions expressible by the neural networks has to be taken into account. Here we chose the discriminator networks as standard $\mathtt{ReLU}$ networks. As such networks also have well studied UAP, the control of the discriminator model error seems to be easy. 
However, there is a fine point to consider. As one usually would like to avoid to control the behavior of neural networks at infinity, it is common to choose source and target measures as compactly supported in, say, $[0,1]^d$. This point is particularly important for $\mathtt{ReQU}$ or $\mathtt{ReCU}$ networks that are not globally Lipschitz. This, however, creates a new technical problem as the density of the source measure is not continuous on the boundary of $[0,1]^d$
and this discontinuity is mapped to some manifold in the vicinity of $[0,1]^d$. 
This leads to discontinuities in the optimal discriminator and to problems with the reduction of the model error by continuous $\mathtt{ReLU}$ networks. Nevertheless, as this discontinuity only occurs at the said manifold, we can we can adopt the approach of \cite{yarotsky2017error} for the universal approximation property (UAP) to obtain good approximation properties everywhere except for an $\varepsilon$-neighborhood of the manifold. Using H. Weyl's tube theorem, we are able to control the volume of the problematic region where the UAP for the discriminator networks with respect to the $\left\|\cdot\right\|_\infty$ norm fails.

By combination of the aforementioned technical elements, we are able to prove our main theorem (Theorem~\ref{theo:main}): 
The approximation error in learning the target distribution with a Brenier GAN, measured by the Jensen–Shannon divergence between the estimated and true distributions, can be made arbitrarily small.
Moreover, we prove that as the sample size grows, the probability of successful learning (i.e., achieving this small error) converges to one, guaranteeing asymptotic consistency of our approach.

We also implement our Brenier GAN architecture and present numerical studies for low dimensional distributions and images of small image size. Our numerical findings show that the convexity loss is mostly becoming inactive after a few epochs of training and only rarely resurfaces. We take this as an numerical indication that our proposed model is in fact able to learn convexity.

Some of the ideas presented here have been explored in various forms in the literature. The statistical learning theory of GAN has been considered by a number of authors~\cite{asatryan2020convenient,biau2020some}. However these works do not learn potential functions. Input convex neural networks (ICNN) have been considered in~\cite{amos2017input}. However~\cite{amos2017input} focuses on reinforcement learning and energy based learning and does not consider UAP of ICNN with respect to convex functions. 
Input convex neural networks with \texttt{ReLU} activation have been shown to possess the universal approximation property for convex functions \cite[Proposition 3]{huangconvex}. However the generator of these can not be trained by backpropagation as these networks are only once (weakly) differentiable. 
More scalable convolutional input convex neural networks are proposed in ~\cite{korotin2019wasserstein}, where the authors study Wasserstein-2-GAN in the context of learning the Kantorovitch potentials and its conjugate. Concerning universal approximation this work relies on~\cite{chen2018optimal}. Despite providing results on the model error under the assumption that the UAP holds, a detailed analysis within the framework of statistical learning theory is still lacking.   
Further recent interesting work concerns the direct parametrization of gradient functions by neural networks, see e.g.~\cite{richter2021input} which is also motivated by modeling gradients of the Brenier potential. The recent paper~\cite{chaudhari2024gradient} proves universal approximation for the gradients of finite convex functions for shallow neural networks with softmax activation in the hidden layer and symmetric weights for input and output layers. While this opens up interesting new perspectives for the future, it is not evident how to generalize these results to deeper networks. 

Our paper is organized as follows. In Section 2, we introduce some notation and the fundamental concept of GANs. In Section 3, we develop the statistical learning theory for Brenier GAN and prove that, with arbitrarily high probability (less than one), the error—measured in terms of the Jensen-Shannon divergence between the distribution of the generative model and the target distribution—can be made arbitrarily small in the large-sample limit.
In Section 4, we introduce a novel penalty function that enforces convexity of the neural network parameterizing the potential. This penalty is based on checking (strong) convexity at randomly drawn samples. We further show that if the penalty weight is chosen sufficiently large, the convexity constraint is satisfied by the minimizer of the loss function.
In Section 5, we present numerical experiments, demonstrating generation results for Gaussian mixtures and small images. Finally, Section 6 concludes the paper and discusses open problems for future research.




\section{Notation and background on GANs}

\subsection{General notation}

Throughout this article we use the following notation. 
We denote by $\Id_d=\Id$ the identity matrix on $\R^d$. For any vectors $v,w\in \R^d$ we denote by $\|v\|\in [0,\infty)$ the Euclidean norm of $v$ and by $\langle v,w\rangle$ the Euclidean scalar product between $v$ and $w$. We refer to the spectral norm of a matrix $A\in \R^{d\times d}$ by $\|A\|_2$ and to its Frobenius norm by $\|A\|_F$.
For two matrices $A,B\in \R^{d\times d}$ we write $A\ge B$ if $\langle v,(A-B)v\rangle \ge 0$ for all $v\in \R^d$.
For a set $A\subseteq \R^d$ we denote by $A^c$ its complement, by $\inte(A)$ its interior, by $\cl(A)$ its closure and by $\partial A$ its boundary.

Moreover, for $f\in C^1([0,1]^d,\R^m)$ we use $(\nabla f)_{i,j}=\frac{\partial f_i}{\partial x_j}$, $i\in \{1,\ldots, m\}$, $j\in \{1,\ldots,d\}$, for the matrix of first derivatives. For $f\in C^2([0,1]^d,\R)$ we write $(\Hess f)_{i,j}=\frac{\partial^2 f}{\partial x_i \partial x_j}$, $i,j\in \{1,\ldots,d\}$, for its Hessian.
For all $d\in \N$ and $\gamma=(\gamma_1,\ldots,\gamma_d) \in \N_0^d$ let $|\gamma|=\sum_{i=1}^d\gamma_i$ and for every function $f=(f_1,\ldots, f_m)\colon \R^d\to \R^m$ that has partial derivatives up to order $|\gamma|$  we denote the partial derivatives by
$$\nabla^\gamma f_i=\frac{\partial^{|\gamma|}f_i}{\partial x_1^{\gamma_1}\cdots \partial x_d^{\gamma_d}},\quad i\in \{1,\ldots,m\}.$$
For all $d,m\in \N$, $k\in \N_0$ and open sets $A\subseteq \R^d$ we denote by $C^k(A,\R^m)$ the set of all functions $f\colon A\to \R^m$ for which $\nabla^\gamma f_i$, $i\in \{1,\ldots,m\}$, are continuous for all $\gamma \in \N_0^d$ with $|\gamma|\le k$. By $C^k(\cl(A),\R^m)$ we denote the set of functions $f\colon \cl(A)\to \R^m$ for which $f|_{A}\in C^k(A,\R^m)$ and whose partial derivatives are uniformly continuous on $A$ (and therefore can be uniquely extended to $\cl(A)$). For $f\in C^k(\cl(A),\R^m)$ we introduce the norm
$$
\|f\|_{C^k(\cl(A),\R^m)}=\max_{\substack{\gamma \in \N_0^d,\\ |\gamma|\le k}}\max_{i\in \{1,\ldots,m\}}\sup_{x\in \cl(A)}|\nabla^\gamma f_i(x)|,
$$
which makes $(C^k(\cl(A),\R^m), \|\cdot\|_{C^k(\cl(A),\R^m)})$ a Banach space.

For $k\in \N_0$, $\alpha \in (0,1]$ and an open set $A\subseteq \R^d$ we denote by $C^{k,\alpha}(\cl(A),\R^m)$ the Hölder space that consists of all functions $f\in C^k(\cl(A),\R^m)$ for which
$$
\|f\|_{C^{k,\alpha}(\cl(A),\R^m)}=\|f\|_{C^k(\cl(A),\R^m)}+
\max_{\substack{\gamma \in \N_0^d,\\ |\gamma|= k}}\max_{i\in \{1,\ldots,m\}}\sup_{\substack{x,y\in A,\\ x\neq y}} \frac{|\nabla^\gamma f_i(x)-\nabla^\gamma f_i(y)|}{\|x-y\|^\alpha}<\infty.
$$

\subsection{Generative framework}

We denote the Kullback-Leibler divergence between a probability measure $\mu$ on $(\R^d, \mathcal B(\R^d))$ (with $\mathcal B(\R^d)$ being the Borel-$\sigma$-algebra on $\R^d$) and a probability measure $\nu$ on $(\R^d, \mathcal B(\R^d))$ that is absolutely continuous w.r.t.\ $\mu$
by
$$
d_{KL}(\nu\Vert\mu)=\int_{\R^d} \log\left(\frac{d\nu}{d\mu}\right)d\nu.
$$
Moreover, for every two probability measures $\mu$, $\nu$ on $(\R^d, \mathcal B(\R^d))$ the Jensen-Shannon divergence between $\mu$ and $\nu$ is given by
$$
d_{JS}(\nu,\mu)=\frac{1}{2}\left( d_{KL}\left(\nu\Big\Vert\frac{\mu+\nu}{2}\right)+d_{KL}\left(\mu\Big\Vert\frac{\mu+\nu}{2}\right)\right).
$$

Throughout this paper, we denote by $\mu^*$ the reference probability measure on $([0,1]^d, \mathcal{B}([0,1]^d))$, for which we aim to learn a mechanism to generate samples. We assume that $\mu^*$ is absolutely continuous with respect to the $d$-dimensional Lebesgue measure and denote its density by $p^*\colon \mathbb{R}^d \to [0,\infty)$. Since $\mu^*$ is supported on $[0,1]^d$, it follows that $p^*(x) = 0$ for all $x \notin [0,1]^d$. Additionally, we assume that $p^*(x) > 0$ for all $x \in [0,1]^d$.
To establish our main result on the approximation error of the Brenier GAN, we impose the following assumption which plays a crucial role in deriving regularity properties of the Brenier potential, essential in our proof.

\begin{assumption}[Regularity of the target measure]\label{assumption:regularity}
There exists $\alpha \in (0,1)$ such that the density $p^*$ of the probability measure $\mu^*$ is in $C^{1, \alpha}([0,1]^d, \mathbb{R})$.
\end{assumption}

Our goal is to learn a transport map $G \colon [0,1]^d \to \mathbb{R}^d$ (also called \emph{generator} in the context of GANs) from a source probability measure $\lambda$ on $([0,1]^d, \mathcal{B}([0,1]^d))$ to $\mu^*$. In this work, we fix $\lambda$ as the restriction of the $d$-dimensional Lebesgue measure to $([0,1]^d, \mathcal{B}([0,1]^d))$.
For learning, we assume access to samples from $\mu^*$. To formalize this, we introduce a probability space $(\Omega, \mathcal{F}, \mathbb{P})$ and define i.i.d.\ random variables $Y, Y_1, Y_2, \dots \colon \Omega \to [0,1]^d$, where $Y$ follows the distribution $\mu^*$. Similarly, we define i.i.d.\ random variables $Z, Z_1, Z_2, \dots \colon \Omega \to [0,1]^d$, where $Z$ follows the distribution $\lambda$.

A \emph{discriminator} $D\colon \R^d\to [0,1]$ is a function which is trained to distinguish between samples from $\mu^*$ and those of $G_{\#}\lambda$. It is supposed to return the value $1$ for samples of $\mu^*$ and $0$ for those of $G_{\#}\lambda$. 

\begin{remark}
    The reference measure $\mu^*$ is only supported on $[0,1]^d$ but samples $G(Z)$ from the generator might take values outside of $[0,1]^d$ (c.f.\ Remark \ref{rem:image_gen}). For this reason the domain of the discriminators is the whole $\R^d$. Since samples outside of $[0,1]^d$ can only come from $G(Z)$, in Section \ref{sec:discriminator_error} we restrict the hypothesis space of discriminators $\Hdis$ to functions $D$ that are equal to $0$ outside of $[0,1]^d$. 
\end{remark}
We next introduce the loss functions for the GAN and Brenier GAN frameworks.

\subsection{GAN and Brenier GAN losses}

We first introduce the standard GAN loss, followed by the Brenier GAN framework, which incorporates optimal transport theory by modeling the generator as the gradient of a convex potential function.

\paragraph{GAN Loss}
For measurable functions $G\colon [0,1]^d\to \mathbb{R}^d$ and $D\colon \R^d\to [0,1]$, the GAN loss is defined as
\begin{equation}
L(G,D) = \frac{1}{2} \mathbb{E} \left[\log D(Y) + \log (1 - D(G(Z))) \right].
\end{equation}
The goal is to solve the  min-max problem
\begin{equation}
    \inf_{G \in \Hgen} \,\sup_{D \in \Hdis}   L(G,D).
\end{equation}
Given a finite sample of size $n$, the empirical counterpart of the GAN loss\footnote{Since the summands in $\hat{L}_n(G,D)$ and the integrands in $L(G,D)$ are non-positive, the sum and expectation are well-defined but may take the value $-\infty$.} is given by
\begin{equation}
\hat{L}_n(G,D) = \frac{1}{2n} \sum_{i=1}^n \log D(Y_i) + \frac{1}{2n} \sum_{i=1}^n \log (1 - D(G(Z_i))).
\end{equation}

\paragraph{Brenier GAN Loss}
The Brenier GAN builds upon the optimal transport framework by modeling the generator as the gradient of a scalar potential function. Specifically, the generator is defined as $G = \nabla \phi$, where $\phi \colon [0,1]^d \to \mathbb{R}$ is a convex potential function parameterized by a neural network. To ensure the learned function adheres to convexity constraints while maintaining the adversarial training dynamics, we introduce a modified loss function that consists of two main components:
(i) the standard GAN loss applied to $\nabla \phi$, ensuring meaningful data generation, and 
(ii) a convexity penalty term enforcing strong convexity of $\phi$. 

For any differentiable function $\phi \colon [0,1]^d \to \mathbb{R}$ and any measurable function $D\colon [0,1]^d \to [0,1]$, the Brenier GAN loss is defined as
\begin{equation}\label{eq:brenier_GAN_loss}
    \mathcal{L}_{\kappa, \gamma}(\phi, D) = L(\nabla \phi, D) + \gamma \hat{P}^{(\kappa)}(\phi),
\end{equation}
where $\hat{P}^{(\kappa)}(\phi)$ is a penalty function designed to promote convexity, $\kappa\in (0,+\infty)$ is the strong convexity parameter and $\gamma \in (0,+\infty)$ is the penalty parameter. We will elaborate further on the nature of the penalty term in Section \ref{sec:convexity}.

\section{Theory of Brenier GAN}\label{sec:error_theory}

In this section, we develop a statistical error theory for the Brenier GAN under the assumption that the hypothesis space of possible generators $G$ is restricted to functions that can be expressed as the gradient of a strongly convex potential. Consequently, we temporarily disregard the convexity penalty term in \eqref{eq:brenier_GAN_loss}. This term will be reintroduced in Section 4, where we will demonstrate that the main result derived in this section remains valid, as the strong convexity of the potential can be effectively enforced through our proposed penalty function.

\subsection{A general framework for generative adversarial learning}

In this subsection we formulate a general framework for generative adversarial learning on which our later results are based. 
 We consider functions $G$ from a hypothesis space $\Hgen$ and aim to minimize the Jensen-Shannon divergence between $\mu^*$ and the pushforward measures $G_{\#}\lambda$, $G\in\Hgen$. Recall that $G_{\#}\lambda$ is the image measure of $G(Z)$, i.e., it satisfies for all $A \in \mathcal B(\R^d)$ that 
$(G_{\#}\lambda)(A)=\lambda(G^{-1}(A))=\mathbb P[(G(Z))^{-1}(A)]$. 
In this subsection we first introduce a generic hypothesis space $\Hgen$ that will be further specified later on.  
We assume that $\Hgen$ is a subset of the set of all measurable functions $G\colon [0,1]^d\to \R^d$
such that 
$G_{\#} \lambda$ is absolutely continuous with respect to the $d$-dimensional Lebesgue measure. For every $G\in\Hgen$ we denote by $D_G\colon \R^d\to [0,1]$ the function that satisfies
$$
D_G(x)=\frac{p^*(x)}{p^*(x)+p(x)}
$$
for all $x\in [0,1]^d$ and $D_G(x)=0$ for all $x\notin [0,1]^d$. Here $p$ is the density of $G_{\#} \lambda$ with respect to the $d$-dimensional Lebesgue measure. We assume that the hypothesis space $\Hdis$ of discriminator functions is a subset of all measurable functions $D\colon \R^d\to [0,1]$.

\begin{remark}\label{rem:image_gen}
    The assumption that $\mu^*$ is supported on the unit cube $[0,1]^d$ is common in the literature and is imposed here mostly for the sake of notational simplicity and could be relaxed to more general, sufficiently regular subsets of $\R^d$. Note, however, that we do not assume that each probability measure $G_{\#}\lambda$ generated by a $G\in\Hgen$ is supported on $[0,1]^d$. We need to allow for more general supports because in Section \ref{sec:generator_error} we take $\Hgen$ as a set of derivatives $\nabla \phi$ of certain neural networks $\phi$ and in this context one cannot guarantee that $\nabla \phi$ only takes values in $[0,1]^d$.
\end{remark}


The next well-known result (see, e.g.,~\cite[Proposition 2.2]{asatryan2020convenient}) relates the Jensen-Shannon divergence to the loss function $L$ and shows the optimality of the discriminator $D_G$ for a given $G\in\Hgen$. Note that we do not require that $D_G\in\Hdis$
\begin{lemma}\label{lem:opt_discriminator}
Let $G\in\Hgen$. Then it holds for all measurable $D\colon \R^d\to [0,1]$ that 
$$
L(G,D)\le L(G,D_G) = d_{JS}(\mu^*,G_{\#} \lambda)-\log(2).
$$
\end{lemma}

We next decompose the divergence $d_{JS}(\mu^*, (\hat G_n)_{\#}\lambda)$ between the reference measure $\mu^*$ and the (random) measure $(\hat G_n)_{\#}\lambda$ learned from the observation of $n\in \N$ samples $Y_i$, $i\in \{1,\ldots,n\}$, into four components: The generator error $\Delta_G=\inf_{G\in\Hgen}d_{JS}(\mu,G_{\#} \lambda)$ accounts for the expressiveness of the generator space $\Hgen$, the discriminator error $\Delta_D=\sup_{G\in\Hgen}\left\{ L(G,D_G)-\sup_{D\in \Hdis}L(G,D)\right\}$ measures how well the optimal discriminator $D_G$ can be approximated within $\Hdis$ uniformly in $\Hgen$, $\Delta_S(n)=\sup_{G\in\Hgen, D\in \Hdis}|L(G,D)-\hat L_n(G,D)|$ is the uniform statistical error for a sample size of $n$ and the training error $\Delta_T(n)=\sup_{D\in \Hdis} \hat L_n(\hat G_n,D)- \inf_{G\in\Hgen}\sup_{D\in \Hdis} \hat L_n(G,D)$ measures how well $\hat G_n$ minimizes the costs $\sup_{D\in \Hdis} \hat L_n(G,D)$ after training.

\begin{propo}\label{prop:error_decomp}
Let $\hat G_n (=\hat G_n(\omega))\in \Hgen$, $n\in \N$, $\omega\in \Omega$, be a random sequence of generators. Let
\begin{align*}
    \Delta_G&=\inf_{G\in\Hgen}d_{JS}(\mu^*,G_{\#} \lambda), & 
\Delta_D&=\sup_{G\in\Hgen}\left\{ L(G,D_G)-\sup_{D\in \Hdis}L(G,D)\right\},\\
\Delta_S(n)&=\sup_{G\in\Hgen, D\in \Hdis}|L(G,D)-\hat L_n(G,D)|, &
\Delta_T(n)&=\sup_{D\in \Hdis} \hat L_n(\hat G_n,D)- \inf_{G\in\Hgen}\sup_{D\in \Hdis} \hat L_n(G,D).
\end{align*}
Then it holds for all $n\in \mathbb N$ a.s.
$$
d_{JS}(\mu^*, (\hat G_n)_{\#}\lambda)\le \Delta_T(n)+
    \Delta_D+2\Delta_S(n)+\Delta_G.
$$

\end{propo}
\begin{proof}
Lemma \ref{lem:opt_discriminator} ensures for all $n\in \mathbb N$ that
    \begin{equation}
    \begin{split}
        L(\hat G_n, D_{\hat G_n})
        &\le \Delta_D + \sup_{D\in \Hdis}L(\hat G_n,D)\\
        &\le \Delta_D+\Delta_S(n)+\sup_{D\in \Hdis}\hat L_n(\hat G_n,D)\\
        &= \Delta_D+\Delta_S(n)+\Delta_T(n) + \inf_{G\in\Hgen}\sup_{D \in \Hdis}\hat L_n(G,D) \\
        &\le \Delta_D+2\Delta_S(n)+\Delta_T(n)+\inf_{G\in\Hgen}\sup_{D \in \Hdis}L(G,D)\\
        &\le\Delta_D+2\Delta_S(n)+\Delta_T(n)+\inf_{G\in\Hgen}L(G,D_G).
    \end{split}
    \end{equation}
This and again Lemma \ref{lem:opt_discriminator} show that
\begin{align}
    d_{JS}(\mu, (\hat G_n)_{\#}\lambda)&=L(\hat G_n, D_{\hat G_n})+\log(2)\nonumber \\
    &\le 
\Delta_T(n)+\Delta_D+2\Delta_S(n)+\inf_{G\in\Hgen}L(G,D_G)+\log(2) \nonumber \\
    &=\Delta_T(n)+\Delta_D+2\Delta_S(n)+\Delta_G.
\end{align}

\end{proof}

In the following sections, we use Proposition \ref{prop:error_decomp} to show that $d_{JS}(\mu,(\hat G_n)_{\#}\lambda)$ can be reduced arbitrarily close to the limit training error $\limsup_{n\to \infty}\Delta_T(n)$ if the hypothesis spaces $\Hgen$ and $\Hdis$ are chosen as appropriate classes of neural networks. 

\subsection{Classes of neural networks}
In the next sections we restrict the hypothesis spaces $\Hgen$ and $\Hdis$ to classes of feedforward deep neural networks. We consider deep neural networks with either Rectified Linear Unit (ReLU) activation function $\sigma(x)=\max\{x,0\}$, $x\in \R$, or Rectified Cubic Unit (ReCU) activation function $\sigma(x)=(\max\{x,0\})^3$, $x\in \R$. In the sequel we denote the componentwise application of $\sigma$ also by $\sigma$, i.e., $\sigma(x)=(\sigma(x_1),\ldots,\sigma(x_d))$ for $x=(x_1,\ldots,x_d)\in \R^d$.

Given $L\in \N$ and $\mathcal A=(l_0,l_1,\ldots,l_L)\in \N^{L+1}$
a neural network $\Phi=((W_1,B_1),(W_2,B_2),\ldots,(W_L,B_L))$ is a $L+1$-tupel of pairs of weight matrices $W_k\in \R^{l_k\times l_{k-1}}$ and bias vectors $B_k\in \R^{l_k}$ which implements the function $f_{\Phi}\colon \R^{l_0}\to \R^{l_L}$ which is for all $x_0\in \R^{l_0}$ defined by $f_\Phi(x_0)=W_Lx_{L-1}+b_L$, where $x_k=\sigma(W_kx_{k-1}+B_k)$, $k\in \{1,\ldots,L-1\}$. 
Throughout this work, we often do not distinguish between a neural network and the function it represents. The tuple $\mathcal{A}$ is referred to as the \textit{network architecture}. The natural number $L$ denotes the number of layers in the neural network, commonly referred to as its \textit{depth}.

\subsection{The generator error $\Delta_G$}\label{sec:generator_error}

We first present some auxiliary results that we need to bound the generator error. 
\begin{lemma}\label{lem:JS_le_dens}
    Let $\mu$ and $\nu$ be two probability measures on $(\R^d, \mathcal B(\R^d))$ with densities $p$ and $q$ with respect to $\lambda$. Then
    $$
d_{JS}(\mu,\nu)\le \frac{\log(2)}{2}\int_{p+q>0} \frac{(p(x)-q(x))^2}{p(x)+q(x)}dx.
$$
\end{lemma}
\begin{proof}
    See, e.g.,~\cite[Lemma 7]{belomestny2021rates}.
\end{proof}

\begin{lemma}\label{lem:exis_bd_dens}
    Let $M\in (0,\infty)$ and $\phi\in C^2([0,1]^d,\R)$ satisfy $(\Hess \phi)(x)\ge \frac{1}{M} \Id$ for all $x\in [0,1]^d$. Then $\nabla \phi\colon [0,1]^d \to \nabla \phi([0,1]^d)$ is bijective and the probability measure $(\nabla \phi)_{\#}\lambda$ has the density
    $$
    p(x)=\frac{1}{|\det[(\Hess \phi)((\nabla \phi)^{-1}(x))]|}, \qquad x\in \nabla \phi([0,1]^d).
    $$
    Moreover, we have $p(x)\le M^d$ for all $x\in \nabla \phi([0,1]^d)$.
    If, additionally, it holds that $(\Hess \phi)(x)\le M \Id$, then we have $p(x)\ge M^{-d}$ for all $x\in \nabla \phi([0,1]^d)$.
\end{lemma}
\begin{proof}
    The claim that $\nabla \phi$ is injective follows e.g.\ from
    $$
\langle x-y,\nabla \phi(x)-\nabla \phi(y)\rangle=\int_0^1\langle x-y, [(\Hess \phi)(y+t(x-y))](x-y)\rangle dt\ge \frac{1}{M}\|x-y\|^2.
$$
The existence and the form of the density $p$ follows from the change of variables theorem.

The upper bound for $p$ follows from
$$
|\det[(\Hess \phi)((\nabla \phi)^{-1}(x))]|=\det[(\Hess \phi)((\nabla \phi)^{-1}(x))]\ge M^{-d}
$$
and the lower bound similarly.
\end{proof}

\begin{lemma}\label{lem:density_le_gen}
    Let $M\in (1,\infty)$ and $\phi_i\in C^2([0,1]^d,\R)$, $i\in \{1,2\}$, satisfy $\frac{1}{M} \Id \le (\Hess \phi_i)(x)\le M \Id $ for all $i\in \{1,2\}$, $x\in [0,1]^d$. Moreover, assume that $\phi_2\in C^{2,1}([0,1]^d,\R)$. Let 
    $$
    p_i(x)=\frac{1}{|\det[(\Hess \phi_i)((\nabla \phi_i)^{-1}(x))]|}, \qquad x\in \nabla \phi_i([0,1]^d), i\in \{1,2\}
    $$
    be the densities of $(\nabla \phi_1)_{\#}\lambda$ and $(\nabla \phi_2)_{\#}\lambda$, respectively (cf.\ Lemma \ref{lem:exis_bd_dens}). 
Then it holds that
$$
\|p_1-p_2\|_{L^\infty(\nabla \phi_1([0,1]^d) \cap \nabla \phi_2([0,1]^d))}\le C \|\phi_1-\phi_2\|_{C^2([0,1]^d)},
$$
where $C\in (0,\infty)$ only depends on $d$, $M$ and $\|\phi_2\|_{C^{2,1}([0,1]^d}$.
\end{lemma} 
\begin{proof}
    See~\cite[Lemma 3]{belomestny2021rates}.
\end{proof}

\begin{propo}\label{propo:existence_nn}
Let $\alpha \in (0,1]$, $d,p\in \N$, $f\in C^{3,\alpha}([0,1]^d, \R^p)$. Then there exists a sequence of deep neural networks $h_n\in C^{2,1}([0,1]^d, \R^p)$, $n\in \N$, with ReCU activation such that 
$$
\|f-h_n\|_{C^{2,1}([0,1]^d,\R^p)}\to 0,
$$
as $n\to \infty$.
\end{propo}
\begin{proof}
The result can be established using the approach outlined in~\cite{belomestny2023simultaneous}, where, additionally, non-asymptotic convergence rates in Hölder norms for deep neural networks with Rectified Quadrat Unit (ReQU) activation, $\sigma(x)=(\max(0,x))^2$, $x\in \R$, are derived. We only sketch the modifications necessary to cover the situation of a ReCU instead of the ReQU activation here. First, note that also $\sigma(x)=(\max(0,x))^3$, $x\in \R$, can represent multiplication exactly. Indeed, the fact that for all $x\in \R$ we have $\sigma(x)-\sigma(-x)=x^3$ implies for all $x,y\in \R$ that 
\begin{multline*}
        24xy=\sigma(x+y+1)-\sigma(-(x+y+1))+\sigma(-(x+y)+1)-\sigma(x+y-1)\\
        -\sigma(x-y+1)+\sigma(-x+y-1)-\sigma(-x+y+1)+\sigma(x-y-1).
    \end{multline*} 
Moreover, since for all $x\in \R$ it holds that
$$
x=\sigma(x+1/\sqrt{6})-\sigma(-x-1/\sqrt{6})+\sigma(x-1/\sqrt{6})-\sigma(1/\sqrt{6}-x)-2\sigma(x)+2\sigma(-x)
$$
it follows that
the real identity can be represented by the ReCU activation.
Furthermore, note that by~\cite[Theorem 4.32 and Eq. (4.46)]{schumaker2007spline} any B-spline of order $3$ can be exactly represented by a deep neural network with ReCU activation. This together with the recursive definition of B-Splines, see e.g.~\cite[Theorem 4.15]{schumaker2007spline}, allows to proceed as in~\cite[Lemma 3]{belomestny2023simultaneous} to show that B-splines of arbitrary order (higher than $3$) can be exactly represented by deep neural networks with ReCU activation. This together with an approximation result for multivariate splines in Hölder norms as e.g.~\cite[Theorem 3]{belomestny2023simultaneous} yields the claim. 
\end{proof}

We next present the main result of this section. For this result, we rely on the regularity Assumption \ref{assumption:regularity} on the density $p^*$ of $\mu^*$ and we will regularly build upon the following remark. 

\begin{remark}\label{rmk:strong_convexity}
    According to Caffarelli's regularity theory~\cite[Theorem 12.50]{villani2008optimal}, if the source and target densities are in $C^{1, \alpha}([0,1]^d, \R)$, then the Brenier potential $\phi^*$ satisfies $\phi^* \in C^{3, \alpha}([0,1]^d, \R)$. In addition, under these same assumptions, the Brenier potential is strongly convex~\cite{caffarelli1996boundary}, which, together with $\phi^* \in C^{3, \alpha}([0,1]^d, \mathbb{R})$, implies that there exists a constant $M\in (1, \infty)$ such that $\frac{1}{M} \mathrm{Id} \leq (\Hess \phi^*)(x) \leq M \, \mathrm{Id}$ for all $x \in [0,1]^d$.
\end{remark}

\begin{propo}\label{prop:main_gen}
Suppose Assumption \ref{assumption:regularity} is satisfied. Let
$M\in (1,\infty)$, $\phi^* \in C^{3,\alpha}([0,1]^d,\R)$ and let $\mu^*=(\nabla \phi^*)_{\#}\lambda$. Let $M>0$ such that $\frac{1}{M}\Id \le (\Hess  \phi^*)(x)\le M \Id$ for all $x\in [0,1]^d$.
 Then for every $\epsilon>0$ there exists a deep neural network $\phi\colon [0,1]^d \to \R$ with ReCU activation such that $\phi\in C^{2,1}([0,1]^d,\R)$, $\frac{1}{2M}\Id \le(\Hess  \phi)(x)\le 2M \Id$ for all $x\in [0,1]^d$
 and $d_{JS}(\mu^*,(\nabla \phi)_{\#}\lambda)\le \epsilon$.
\end{propo}
\begin{proof}
Let $\epsilon >0$.
    By Proposition \ref{propo:existence_nn} there exists a sequence $\phi_n \in C^{2,1}([0,1]^d,\R)$, $n\in \N$, of deep neural networks with ReCU activation such that 
\begin{equation}\label{eq:conv_dnn_G}
    \|\phi^*-\phi_n\|_{C^{2,1}([0,1]^d,\R)}\to 0.
\end{equation}
as $n\to \infty$.
In particular, it holds that $\phi_n\in C^{2,1}([0,1]^d,\R)\subseteq C^2([0,1]^d,\R)$ for all $n\in \N$. 

By \eqref{eq:conv_dnn_G} the sequence $\nabla \phi_n$, $n\in \N$, converges pointwise to $\nabla \phi^*$. 
This implies that for all $x\in [0,1]^d$ with $\nabla \phi^*(x)\in (0,1)^d$ we have for all sufficiently large $n\in \N$ that $\nabla \phi_n(x)\in [0,1]^d$. Dominated convergence thus ensures that
$$
\lim_{n\to \infty}\int_{[0,1]^d} \mathbbm{1}_{(0,1)^d}(\nabla \phi^*(x)) \mathbbm{1}_{[0,1]^d}(\nabla \phi_n(x)) dx
=\int_{[0,1]^d} \mathbbm{1}_{(0,1)^d}(\nabla \phi^*(x))  dx = \mu^*((0,1)^d)=1.
$$
Hence we can choose $n\in \N$ large enough so that 
\begin{equation}\label{eq:n_large_1}
\int_{[0,1]^d} \mathbbm{1}_{[0,1]^d}(\nabla \phi_n(x))dx \ge 1-\frac{\epsilon}{2\log(2)}.
\end{equation}
Moreover, \eqref{eq:conv_dnn_G} ensures that by potentially increasing $n$ we have 
\begin{equation}\label{eq:n_large_2}
\sup_{x\in [0,1]^d}\|(\Hess \phi^*)(x)-(\Hess \phi_n)(x)\|_2\le \frac{1}{2M}.
\end{equation}
Finally, again by \eqref{eq:conv_dnn_G} and by potentially increasing $n$ once more we have 
\begin{equation}\label{eq:n_large_3}
\|\phi^*-\phi_n\|_{C^2([0,1]^d)}\le \frac{\epsilon}{2\log(2)C},    
\end{equation}
where $C$ is the constant from Lemma \ref{lem:density_le_gen} (with $M$ replaced by $1/(2M)$ and $\|\phi_2\|_{C^{2,1}([0,1]^d)}$ replaced by $\|\phi^*\|_{C^{2,1}([0,1]^d)}$). In the remainder of the proof we write $\phi=\phi_{n}$.

Note that \eqref{eq:n_large_2} and the assumption that $\frac{1}{M}\Id \le(\Hess  \phi^*)(x)\le M \Id$ ensure for all $x\in [0,1]^d$, $v\in \R^d$
\begin{align*}
        v^T[(\Hess \phi)(x)]v&=v^T[(\Hess \phi^*)(x)]v+v^T[(\Hess \phi)(x)-(\Hess \phi^*)(x)]v\\ &\le (M+\|(\Hess \phi^*)(x)-(\Hess \phi)(x)\|_2)\|v\|^2 \\
&\le \left(M+\frac{1}{2M}\right)\|v\|^2 \\&\le 2M\|v\|^2 ,
\end{align*}
and
\begin{align*}
        v^T[(\Hess \phi)(x)]v&=v^T[(\Hess \phi^*)(x)]v+v^T[(\Hess \phi)(x)-(\Hess \phi^*)(x)]v\\ &\ge \left(\frac{1}{M}-\|(\Hess \phi^*)(x)-(\Hess \phi)(x)\|_2\right)\|v\|^2 \\
&\ge \left(\frac{1}{M}-\frac{1}{2M}\right)\|v\|^2 \\&= \frac{1}{2M}\|v\|^2.
\end{align*}
Hence, for all $x\in [0,1]^d$ we have $\frac{1}{2M}\Id \le(\Hess  \phi)(x)\le 2M \Id$. 

Denote by $p^*$ and $p$ the densities of $(\nabla \phi^*)_{\#}\lambda$ and $(\nabla \phi)_{\#}\lambda$, respectively (cf. Lemma \ref{lem:exis_bd_dens}) and let $A=(\nabla \phi)([0,1]^d)\subseteq \R^d$.
Note that Lemma \ref{lem:JS_le_dens} ensures that 
\begin{equation}\label{eq:aux_gen_error_1}
\begin{split}
    d_{JS}(\mu^*,(\nabla \phi)_{\#}\lambda)&\le \frac{\log(2)}{2}\int_{p^*+p>0} \frac{(p^*(x)-p(x))^2}{p^*(x)+p(x)}dx
\\&
=\frac{\log(2)}{2}\left(\int_{[0,1]^d \cap A} \frac{(p^*(x)-p(x))^2}{p^*(x)+p(x)}dx+\int_{[0,1]^d\setminus A} p^*(x)dx+\int_{A\setminus [0,1]^d} p(x)dx\right)
\\&=
\frac{\log(2)}{2}\left(\int_{[0,1]^d \cap A} \frac{(p^*(x)-p(x))^2}{p^*(x)+p(x)}dx+
\int_{[0,1]^d\cap A} (p(x)-p^*(x))dx
+2\int_{A\setminus [0,1]^d} p(x)dx\right)\\
&\le \log(2)\left(
\int_{[0,1]^d\cap A} |p(x)-p^*(x)|dx
+\int_{A\setminus [0,1]^d} p(x)dx\right).
\end{split}
\end{equation}
Note that \eqref{eq:n_large_1} ensures that
\begin{equation}\label{eq:aux_gen_error_2}
    \int_{A\setminus [0,1]^d} p(x)dx=1-\int_{[0,1]^d\cap A} p(x)dx
=1-\int_{[0,1]^d} \mathbbm{1}_{[0,1]^d}(\nabla \phi(x)) dx \le \frac{\epsilon}{2\log(2)}.
\end{equation}
Moreover, Lemma \ref{lem:density_le_gen} and \eqref{eq:n_large_3} ensure that
\begin{equation}\label{eq:aux_gen_error_3}
 \int_{[0,1]^d\cap A} |p(x)-p^*(x)|dx\le 
\|p^*-p\|_{L^\infty([0,1]^d \cap A) }\le C \|\phi^*-\phi\|_{C^2([0,1]^d)}\le \frac{\epsilon}{2\log(2)}.   
\end{equation}
Combining \eqref{eq:aux_gen_error_1}, \eqref{eq:aux_gen_error_2} and \eqref{eq:aux_gen_error_3} proves that $d_{JS}(\mu^*,(\nabla \phi)_{\#}\lambda)\le \epsilon$ and completes the proof.
\end{proof}

\subsection{The discriminator error $\Delta_D$}\label{sec:discriminator_error}

\begin{assumption}\label{assump:generator}
Suppose Assumption \ref{assumption:regularity} is satisfied. Let $\phi^* \in C^{3,\alpha}([0,1]^d,\R)$ be the Brenier potential, such that $\mu^*=(\nabla \phi^*)_{\#}\lambda$. Let $M\in (1,\infty)$ such that $\frac{1}{M}\Id \le (\Hess  \phi^*)(x)\le M \Id$ for all $x\in [0,1]^d$ (see Remark \ref{rmk:strong_convexity}).
  
  We denote for all $\epsilon>0$ by $\mathcal A^\phi(\epsilon)$ the architecture of the neural network $\phi$ with ReCU activation function constructed in Proposition \ref{prop:main_gen} and by $H=\|\phi\|_{C^{2,1}([0,1]^d,\R)}$ its $C^{2,1}$-norm. 
  
  We assume that for all $\epsilon>0$ there exist $\tilde M\ge 2M$ and $\tilde H\ge H$ such that
  the hypothesis space $\Hgen(\epsilon)$
  of generators consists of all functions $G\colon [0,1]^d\to \R^d$ such that there exists a neural network $\phi\colon [0,1]^d \to \R$ with ReCU activation function\footnote{Note that every neural network with ReCu activation is automatically of regularity $C^{2,1}$.} and neural network architecture $\mathcal A^\phi(\epsilon)$ such that $G=\nabla \phi$, $\|\phi\|_{C^{2,1}([0,1]^d,\R)}\le \tilde H$ and $\frac{1}{\tilde M}\Id \le(\Hess  \phi)(x)\le \tilde M \Id$ for all $x\in [0,1]^d$.

\end{assumption}

In particular, we have under Assumption \ref{assump:generator} for all $\epsilon>0$ that $\Hgen(\epsilon)$ contains the function $G=\nabla \phi$ from Proposition~\ref{prop:main_gen} and consequently
$$\inf_{G\in\Hgen(\epsilon)}d_{JS}(\mu^*,G_{\#}\lambda)\le \epsilon.$$

We next turn to the discriminator error. 

\begin{propo}\label{prop:main_disc}
Suppose that Assumptions \ref{assumption:regularity} and \ref{assump:generator} are satisfied. 
Then there exist for every $\epsilon>0$ a network architecture $\mathcal A^D(\epsilon)$ and $D_{\min}(\epsilon)\in (0,1/2)$ such that for all $G\in\Hgen(\epsilon)$ there exists $D=D(G)\colon \R^d\to [0,1]$ 
such that $D|_{\R^d\setminus[0,1]^d}=0$, $D|_{[0,1]^d}$ is a neural network with ReLU activation function and network architecture $\mathcal A^D(\epsilon)$ and we have $D|_{[0,1]^d}\colon [0,1]^d\to [D_{\min}(\epsilon),1-D_{\min}(\epsilon)]$ and
    $$
    L(G,D_G)-L(G,D)\le \epsilon.
    $$
Moreover, $D|_{[0,1]^d}$ is Lipschitz continuous with Lipschitz constant which depends on $\epsilon>0$ but is independent of $G\in\Hgen(\epsilon)$.
\end{propo}
\begin{proof}
    Throughout the proof we fix $\epsilon>0$ and a generator $G=\nabla \phi\in \Hgen(\epsilon)$. We denote by $A=G([0,1]^d)$ the range of $G$ which is a compact subset of $\R^d$. Recall that $\mu^*=G^*_{\#}\lambda$ with $G^*=\nabla \phi^*\colon [0,1]^d\to \R^d$ and $A^*=G^*([0,1]^d)=[0,1]^d$. The optimal discriminator $D_G\colon \R^d\to [0,1]$ which maximizes $L(G,D)$ satisfies $D_G(x)=\frac{p^*(x)}{p^*(x)+p(x)}$, $x\in A^*$, where $p^*$ and $p$ are the densities of $G^*_{\#}\lambda$ and $G_{\#}\lambda$, respectively (see Lemma \ref{lem:opt_discriminator}). 
    In particular, for all $x\in A\cap A^*$ we have 
    $$
    D_G(x)=\frac{\det[(\Hess \phi)((\nabla \phi)^{-1}(x))]}{\det[(\Hess \phi)((\nabla \phi)^{-1}(x))]+\det[(\Hess \phi^*)((\nabla \phi^*)^{-1}(x))]}
    $$
    and for all $x\in A^*\setminus A$ we have $D_G(x)=1$. 
The assumptions that $\phi, \phi^*\in C^{2,1}([0,1]^d,\R)$, that $\frac{1}{\tilde M}\Id\le (\Hess \phi)(x),(\Hess \phi^*)(x)\le \tilde M \Id$ and that $A$ and $A^*$ are compact ensure that $D_G\in C^{0,1}(\inte(A)\cap A^*,\R)$. 
To summarize, we have that $D_G$ is Lipschitz continuous on $\inte(A)\cap A^*$ (and also $A^*\setminus A$) but is discontinuous at the boundary points $(\partial A)\cap A^*$. We follow the proof of~\cite[Theorem 1]{yarotsky2017error} to construct a discriminator $D$ that uniformly approximates $D_G$ except for a sufficiently small tubular neighborhood of $(\partial A)\cap A^*$.

To this end we fix $N\in \N$ (which is specified later) and decompose the unit cube $A^*=[0,1]^d$ into a regular grid with mesh width $1/N$ and grid points $\mathcal M(N)=\{0,1/N,2/N,\ldots,1\}^d$.
We consider the partition of unity $\rho_m$, $m\in \mathcal M(N)$, as given in ~\cite[Equation (5)]{yarotsky2017error}. 
Moreover, for each $m\in \mathcal M(N)\cap \inte(A)$ we set $P_m=D_G(m)$,
for each $m\in \mathcal M(N)\cap A^c$ we set
$P_m=1$ and for each $m\in \mathcal M(N)\cap (\partial A)$ we set $P_m=0$. Next we define $\tilde D=\sum_{m\in \mathcal M(N)}\rho_m P_m$. Let $\mathcal C(N)$ be the $\sqrt{d}/N$-tubular neighborhood of $(\partial A)\cap A^*$, i.e., $\mathcal C(N)=\cup_{a\in (\partial A)\cap A^*} \{x\in A^* \, | \, \|x-a\|\le \sqrt{d}/N\}$. For $m=(m_1,\ldots,m_d)\in \mathcal M(N)$ recall from ~\cite[Equation (7)]{yarotsky2017error}) that $x=(x_1,\ldots,x_d)\in \operatorname{supp}(\rho_m)$ implies that $|x_k-m_k|<1/N$ for all $k\in \{1,\ldots,d\}$. Hence, for all $x\in \operatorname{supp}(\rho_m)$ we have $\|x-m\|< \sqrt{d}/N$. Consequently, we obtain for all $x\in (\inte{A})\cap A^*$ with $x\notin \mathcal C(N)$ and $x\in \operatorname{supp}(\rho_m)$ that $m\in (\inte{A})\cap A^*$. Likewise, we obtain for all $x\in A^*\setminus A$ with $x\notin \mathcal C(N)$ and $x\in \operatorname{supp}(\rho_m)$ that $m\in A^*\setminus A$. By following the arguments of the proof of~\cite[Theorem 1]{yarotsky2017error} we thus obtain for all $x\in A^*\setminus \mathcal C(N)$ that
\begin{equation}\label{eq:ub_Dg_Taylor}
    |D_G(x)-\tilde D(x)|
\le \frac{C}{N},
\end{equation}
where the constant $C>0$ only depends on $d$, $k$, $\tilde H$ and $\tilde M$. 

Next, we further follow the proof of~\cite[Theorem 1]{yarotsky2017error} to obtain a ReLU-neural network architecture that is capable of uniformly approximating $\tilde D$. More precisely, the proof of~\cite[Theorem 1]{yarotsky2017error} shows that there exists a ReLU-neural network architecture (which depends on $d$, $k$, $\tilde H$, $\tilde M$ and $N$ but not on $G$ or $D_G$)
which contains a neural network $\overline D\colon \R^d\to \R$ such that $\sup_{x\in [0,1]^d}|\tilde D(x)-\overline D(x)|\le \frac{C}{N}$.
Combining this with \eqref{eq:ub_Dg_Taylor} we obtain
\begin{equation}\label{eq:ub_sup_wo_tube}
    \sup_{x\in A^*\setminus \mathcal C(N)}|D_G(x)-\overline D(x)|\le \frac{2C}{N}.
\end{equation}

To complete the construction of the network architecture we introduce a ReLU clipping layer that restricts outputs to $[D_{\min},1-D_{\min}]$, where $0<D_{\min}<\frac{1}{1+M^{2d}}$ is specified later. We denote the resulting neural network with ReLU activation function by $D=(\overline D \wedge (1-D_{\min}))\vee D_{\min}$ on $[0,1]^d$ and set $D=0$ on $\R^d\setminus [0,1]^d$.

Note that by the construction in the proof of~\cite[Theorem 1]{yarotsky2017error} we have that $\overline D$ is of the form $\overline D(x)=\sum_{m\in \mathcal M(N)}\overline \phi_m(x) P_m$ for all $[0,1]^d$. Here,  for every $m\in \mathcal M(N)$ the function $\overline \phi_m$ is a neural network with ReLU activation function and is independent of $G$. Indeed, as the proof of~\cite[Theorem 1]{yarotsky2017error} reveals, the functions $\overline \phi_m$, $m\in \mathcal M(N)$, are neural network approximations of the partition of unity $ \phi_m$, $m\in \mathcal M(N)$, which is chosen independently of $G$. This implies that for every $m\in \mathcal M(N)$ the function $\overline \phi_m$ is Lipschitz continuous with Lipschitz constant independent of $G$. This together with the fact that $P_m\in [0,1]$ for all $m\in \mathcal M(N)$ shows that $\overline D$ is Lipschitz continuous on $[0,1]^d$ with Lipschitz constant independent of $G$. Consequently, the same statement applies to $D$.

We next analyze the error $L(G,D_G)-L(G,D)$. To this end note that 
\begin{equation*}
    \begin{split}
        L(G,D_G)-L(G,D)&=
\frac{1}{2}\mathbb E\left[\log(D_G(Y))+\log(1-D_G(G(Z)))\right]
-\frac{1}{2}\mathbb E\left[\log(D(Y))+\log(1-D(G(Z)))\right]\\
&=\frac{1}{2}\mathbb E\left[\log\left(\frac{D_G(Y)}{D(Y)}\right)+\log\left(\frac{1-D_G(G(Z))}{1-D(G(Z))}\right)\right]\\
&=
\frac{1}{2}\int_{A^*}\log\left(\frac{D_G(y)}{D(y)}\right)p^*(y)dy
+
\frac{1}{2}\int_{A}\log\left(\frac{1-D_G(y)}{1-D(y)}\right)p(y)dy.
    \end{split}
\end{equation*}
Since $D(x)=D_G(x)=0$ for all $x\in A\setminus A^*$ and $D_G(x)=1$ for all $x\in A^*\setminus A$ we obtain
\begin{equation}\label{eq:disc_error_3_int}
    \begin{split}
        L(G,D_G)-L(G,D)&=
\frac{1}{2}\int_{A^*\setminus (A\cup \mathcal C(N))}\log\left(\frac{1}{D(y)}\right)p^*(y)dy
+
\frac{1}{2}\int_{\mathcal C(N)}\log\left(\frac{D_G(y)}{D(y)}\right)p^*(y)dy\\ & \quad +\frac{1}{2}\int_{\mathcal C(N)}\log\left(\frac{1-D_G(y)}{1-D(y)}\right)p(y)dy\\
     &\quad +   
\frac{1}{2}\int_{(A\cap A^*)\setminus \mathcal C(N)}\log\left(\frac{D_G(y)}{D(y)}\right)p^*(y)+\log\left(\frac{1-D_G(y)}{1-D(y)}\right)p(y)dy.
    \end{split}
\end{equation}
We consider the three integrals separately. First note that by \eqref{eq:ub_sup_wo_tube} we have $D(x)\ge 1-\max\{D_{\min},\frac{2C}{N}\}$ for all $x\in A^*\setminus (A\cup \mathcal C(N))$ and hence
\begin{equation}\label{eq:boudn_int_1}
   \frac{1}{2}\int_{A^*\setminus (A\cup \mathcal C(N))}\log\left(\frac{1}{D(y)}\right)p^*(y)dy\le 
   -\frac{1}{2}\log\left(1-\max\left\{D_{\min},\frac{2C}{N}\right\} \right).
\end{equation}
By Lemma \ref{lem:exis_bd_dens} we have $p,p^*\le \tilde M^d$ and hence
\begin{equation*}
\begin{split}
   \frac{1}{2}\int_{\mathcal C(N)}\log\left(\frac{D_G(y)}{D(y)}\right)p^*(y)+\log\left(\frac{1-D_G(y)}{1-D(y)}\right)p(y)dy&\le -\log(D_{\min})\int_{\mathcal C(N)}\frac{p^*(y)+p(y)}{2}dy\\&
   \le -\log(D_{\min}) \tilde M^d \lambda(\mathcal C(N)).
   \end{split}
\end{equation*}
Weyl's tube formula~\cite{weyl1939volume} (see also~\cite{gray2003tubes}) shows that $\lambda(\mathcal C(N))$ is a polynomial in $1/N$ with vanishing zero order term. The fact that $\|\phi\|_{C^{2,1}([0,1]^d,\R)}\le \tilde H$ further entails bounds on the volume element and curvature tensors of the manifolds forming $(\partial A)\cap A^*$. Consequently the coefficients of this polynomial are bounded by some constant that only depends on $d$ and $\tilde H$. In particular, it follows that $\lambda(\mathcal C(N))\le \overline C /N$ for some constant $\overline C$ only depending on $d$ and $\tilde H$ (and not depending on $G$). 
This implies that 
\begin{equation}\label{eq:boudn_int_2}
\begin{split}
   \frac{1}{2}\int_{\mathcal C(N)}\log\left(\frac{D_G(y)}{D(y)}\right)p^*(y)+\log\left(\frac{1-D_G(y)}{1-D(y)}\right)p(y)dy&\le  - \frac{\log(D_{\min}) M^d\overline C}{N}.
   \end{split}
\end{equation}
Finally, note that Lemma \ref{lem:exis_bd_dens} ensures that $D_G(x)\in [\frac{1}{1+M^{2d}},1-\frac{1}{1+M^{2d}}] \subseteq [D_{\min},1-D_{\min}]$ for all $x\in A\cap A^*$. This together with \eqref{eq:ub_sup_wo_tube} entails 
that 
\begin{equation*}
    \sup_{x\in (A\cap A^*)\setminus \mathcal C(N)}|D_G(x)- D(x)|\le \frac{2C}{N}
\end{equation*}
Lipschitz continuity of $[D_{\min},1]\ni x\mapsto\log(x)\in \R$ with constant $1/D_{\min}$ implies that
\begin{equation}\label{eq:boudn_int_3}
\begin{split}
   &\frac{1}{2}\int_{(A\cap A^*)\setminus \mathcal C(N)}\log\left(\frac{D_G(y)}{D(y)}\right)p^*(y)+\log\left(\frac{1-D_G(y)}{1-D(y)}\right)p(y)dy\\
   &\le \frac{1}{D_{\min}}\int_{(A\cap A^*)\setminus \mathcal C(N)}|D_G(y)-D(y)|\frac{p^*(y)+p(y)}{2}dy
   \le 
   \frac{2C}{D_{\min} N}.
   \end{split}
\end{equation}
Combining \eqref{eq:boudn_int_1}, \eqref{eq:boudn_int_2} and \eqref{eq:boudn_int_3} with \eqref{eq:disc_error_3_int} yields
$$
L(G,D_G)-L(G,D)\le 
-\frac{1}{2}\log\left(1-\max\left\{D_{\min},\frac{2C}{N}\right\} \right)
- \frac{\log(D_{\min}) M^d\overline C}{N}
+\frac{2C}{D_{\min} N}.
$$
We thus conclude that by choosing first $D_{\min}$ small enough (e.g., $D_{\min}\le 1-e^{-\epsilon}$) and then $N$ large enough there exists a neural network architecture (independent of $G$) that contains a neural network $D\colon \R^d \to [D_{\min},1-D_{\min}]$ such that
$$
L(G,D_G)-L(G,D)\le \epsilon.
$$
This completes the proof.

\end{proof}

\begin{assumption}\label{assump:disc}
Suppose that Assumptions \ref{assumption:regularity} and \ref{assump:generator} are satisfied. 
For every $\epsilon>0$ let $\mathcal A^D(\epsilon)$ and $D_{\min}(\epsilon)$ denote the network architecture and the cut-off from Proposition \ref{prop:main_disc}, respectively.
We assume that for all $\epsilon>0$ the hypothesis space of discriminators $\Hdis(\epsilon)$ is a relatively compact (with respect to the supremum norm on $[0,1]^d$) subset of the functions $D\colon \R^d \to [0,1]$ such that $D|_{\R^d\setminus[0,1]^d}=0$, $D|_{[0,1]^d}$ is a neural network with ReLU activation function and network architecture $\mathcal A^D(\epsilon)$ and it holds that $D|_{[0,1]^d}\colon [0,1]^d\to [D_{\min}(\epsilon),1-D_{\min}(\epsilon)]$. Moreover, we assume that for every $G\in\Hgen(\epsilon)$ the function $D(G)$ constructed in Proposition \ref{prop:main_disc} is contained in 
$\Hdis(\epsilon)$ (i.e.\ $\{D(G),\,G\in\Hgen(\epsilon)\}\subseteq \Hdis(\epsilon)$).
\end{assumption}

By Proposition \ref{prop:main_disc} we have under Assumption \ref{assump:disc} that for every $\epsilon>0$ 
$$
\sup_{G\in\Hgen(\epsilon,M,H)}\inf_{D\in \Hdis(\epsilon)}L(G,D_G)-L(G,D)\le \epsilon.
$$
\begin{remark}
    In Assumption \ref{assump:disc} we need to impose relative compactness on $\Hdis(\epsilon)$ (as opposed to the situation in Assumption \ref{assump:generator} for $\Hgen(\epsilon)$, where this comes from higher order regularity). Note, however, that the requirement that for every $G\in\Hgen(\epsilon)$ the function $D(G)$ constructed in Proposition \ref{prop:main_disc} is contained in 
$\Hdis(\epsilon)$ does not conflict with the requirement of relative compactness of $\Hdis(\epsilon)$. Indeed, as shown in Proposition \ref{prop:main_disc} for each fixed $\epsilon>0$ the class of functions $\{D(G)|G\in\Hgen(\epsilon)\}\subseteq \Hdis(\epsilon)$ is uniformly Lipschitz continuous and therefore relatively compact. 
\end{remark}

\subsection{The sample error}

\begin{propo}\label{propo:uslln}
    Suppose that Assumption \ref{assump:disc} is satisfied and let $\epsilon>0$. Then it holds a.s.\ that 
    \begin{equation}\label{eq:uslln}
        \lim_{n\to \infty}\sup_{(G,D)\in \Hgen(\epsilon) \times \Hgen(\epsilon)} \left|
\hat L_n(G,D)-L(G,D)
\right|=0.
    \end{equation}
\end{propo}
\begin{proof}
We aim at applying a version of the uniform law of large numbers as stated in, e.g., \cite[Theorem 16]{ferguson2017course}. To this end we first note by Assumption \ref{assump:disc} that $\Hdis(\epsilon)$ is relatively compact in the space of functions that vanish on $\R^d\setminus [0,1]^d$ and are continuous on $[0,1]^d$ equipped with the uniform norm on $[0,1]^d$.
Moreover, note that Assumption \ref{assump:generator} and the Arzelà-Ascoli theorem ensure that $\Hgen(\epsilon)$ is relatively compact in $C^1([0,1]^d,\R^d)$. We conclude that the closure $\mathcal C$ of  $\Hgen(\epsilon)\times \Hdis(\epsilon)$ is compact. 
Next, for $y,z\in [0,1]^d$ and $(G,D)\in \mathcal C$ consider
$$
U(y,z,G,D)=\frac{1}{2}\left[\log(D(y))+\log(1-D(g(z)))\right].
$$
We show that $U(Y,Z,\cdot,\cdot)$ is continuous at every $(G,D)\in \mathcal C$ almost surely. 
To this end let $(G_n,D_n)\in \mathcal C$, $n\in \N$, be a sequence with $(G_n,D_n)\to (G,D)$ as $n\to \infty$. Note that by Assumption \ref{assump:generator} the distribution $G_{\#}\lambda$ of $G(Z)$ is absolutely continuous with respect to the Lebesgue measure. This entails that 
$\partial [0,1]^d$ is a null set of $G_{\#}\lambda$. Combining this 
and the facts that $D_n$ converges uniformly to $D$ on $[0,1]^d$ and that $G_n(Z)\to G(Z)$ a.s.\ we obtain
$D_n(G_n(Z))\to D(G(Z))$ a.s. Using that $D_n$, $n\in\N$, and $D$ take only values in $[D_{\min}(\epsilon),1-D_{\min}(\epsilon)]$ on $[0,1]^d$ ensures that a.s.
$$
\lim_{n\to \infty}U(Y,Z,G_n,D_n)=U(Y,Z,G,D).
$$
Moreover, we obtain that $U$ is uniformly bounded by using once more that $D_n$, $n\in\N$, and $D$ take only values in $[D_{\min}(\epsilon),1-D_{\min}(\epsilon)]$ on $[0,1]^d$. Hence the assumptions of the uniform law of large numbers are satisfied and we conclude that a.s.
$$
\sup_{(G,D)\in \mathcal C}\left|\frac{1}{n}\sum_{j=1}^n U(Y_j,Z_j,G,D)-\mathbb E[U(Y,Z,G,D)]\right|\to 0, \qquad n\to \infty,
$$
which entails \eqref{eq:uslln}.
\end{proof}

\subsection{Consistency}

We combine Proposition \ref{prop:main_gen}, Proposition \ref{prop:main_disc} and Proposition \ref{propo:uslln} with Proposition \ref{prop:error_decomp} to obtain the following result, which states that by choosing the hypothesis spaces according to Assumption \ref{assump:generator} and Assumption \ref{assump:disc} the Jensen-Shannon divergence between the reference measure $\mu^*$ and a measure $(\hat G_n)_{\#}\lambda$ obtained from training on $n\in \N$ samples can be reduced arbitrary close to the training error in the limit $n\to \infty$.

\begin{theorem}\label{theo:main}
    Suppose that Assumption \ref{assump:disc} are satisfied. Let $\epsilon>0$ and $\hat G_n(=\hat G_n(\omega))\in \Hgen(\epsilon)$, $n\in \N$, $\omega \in \Omega$, be a random sequence of generators. Denote by 
    $$
    \Delta_T(n)=\sup_{D\in \Hdis(\epsilon)}\hat L_n(\hat G_n,D)-\inf_{G\in\Hgen(\epsilon)}\sup_{D\in \Hdis(\epsilon)}\hat L_n(G,D)
    $$
    the training error of $\hat G_n$.
    Then it holds a.s.\ that
    $$
    \limsup_{n\to \infty}d_{JS}(\mu^*,(\hat G_n)_{\#}\lambda)\le \limsup_{n\to \infty}\Delta_T(n) +2\epsilon.
    $$
\end{theorem}




\begin{remark}
    In \cite[Theorem 4]{huangconvex} it is proven that for a sequence of functions $\hat\phi_n$ such that $(\nabla \hat\phi_{n})_\#\lambda\rightharpoonup \mu^*=(\nabla \phi^*)_\#\lambda$ (weak convergence), $\nabla\phi_n\to \nabla\phi^*$ holds almost everywhere on $[0,1]^d$, where $\phi^*$ is the Brenier potential. In our analysis we have proven  $d_{JS}((\nabla\hat\phi_{n})_\# \lambda,\mu^*)\to 0$ a.s.\ which implies a.s.\ convergence in law. We thus conclude, as $[0,1]^d$ is star shaped, that we actually a.s.\ learn the Brenier potential (up to an irrelevant constant).  
\end{remark}


\section{Enforcing convexity of the Brenier potential}\label{sec:convexity}

In the previous section, we provided a statistical error analysis for the Brenier GAN (Theorem~\ref{theo:main}), where we showed that minimizing the GAN loss over the hypothesis classes $\Hdis$ and $\Hgen$—the latter composed of generators expressible as gradients of strongly convex potentials—leads to an arbitrarily small error on the estimated distribution $(\hat{G}_n)_{\#}\lambda$ as the number of training samples goes to infinity. In this section, we relax the hypothesis space to the set of generators $G$ expressible as $G = \nabla \phi$, with $\phi \colon [0,1]^d \to \mathbb{R}$, without assuming convexity of $\phi$. As a trade-off, we introduce a penalty term in the loss function to enforce strong convexity of the learned potential. We show that for a sufficiently large penalization parameter, the potential minimizing the penalized loss is indeed strongly convex, allowing us to ensure that the theory developed in the previous section remains valid. 

Next, we introduce the penalty term at the core of our convexity-promoting strategy.

\subsection{Notation}
For any differentiable function $\phi \colon [0,1]^d \to \mathbb{R}$ and any measurable function $D\colon [0,1]^d \to [0,1]$, the Brenier GAN loss with a convexity-promoting penalty is defined by
\begin{equation}
    \mathcal{L}_{\kappa, \gamma}(\phi, D) = L(\nabla \phi, D) + \gamma P^{(\kappa)}(\phi),
\end{equation}
where the penalty term $P^{(\kappa)}(\phi)$ encourages $\kappa$-strong convexity of $\phi$, and is defined as
\begin{align}
    P^{(\kappa)}(\phi) 
    &= \mathbb{E}\left[ \mathrm{ReLU}\left(
    \phi\left(\frac{U+U'}{2}\right) - \frac{\phi(U) + \phi(U')}{2} + \frac{\kappa}{8}\|U - U'\|^2
    \right)\right]\nonumber\\[6pt]
    &= \int_{([0,1]^d)^2} \mathrm{ReLU}\left(
    \phi\left(\frac{u+u'}{2}\right) - \frac{\phi(u) + \phi(u')}{2} + \frac{\kappa}{8}\|u - u'\|^2
    \right)\,du\,du'.
\end{align}
Here $U$ and $U'$ are independent and uniformly distributed on $[0,1]^d$.

This penalty promotes $\kappa$-strong convexity, which is equivalent to requiring that the function $\phi + \frac{\kappa}{2}\|\cdot\|^2$ is convex. Equivalently, $\kappa$-strong convexity can be characterized by the inequality  
$$
\phi\left(\frac{u + u'}{2}\right) \leq \frac{\phi(u) + \phi(u')}{2} - \frac{\kappa}{8} \|u - u'\|^2
$$
for all $u, u' \in [0,1]^d$.


In practice, we estimate this penalty empirically using a finite number of samples. Let $m : \mathbb{N} \to \mathbb{N}$ be an increasing function such that $\lim_{n\rightarrow +\infty} m(n) = +\infty$. Consider $U_1,\dots,U_{m(n)}$ and $U'_1,\dots,U'_{m(n)}$ as independent samples drawn according to $\lambda$. The empirical Brenier GAN loss is defined as
\begin{equation}
    \hat{\mathcal{L}}_{n, \kappa, \gamma}(\phi, D) = \hat{L}_n(\nabla \phi, D) + \gamma\hat{P}^{(\kappa)}_n(\phi),
\end{equation}
where the empirical convexity penalty is given by
\begin{equation}
    \hat{P}^{(\kappa)}_n(\phi) = \frac{1}{m(n)} \sum_{j=1}^{m(n)} \mathrm{ReLU}\left(
    \phi\left(\frac{U_j+U_j'}{2}\right) - \frac{\phi(U_j) + \phi(U_j')}{2} + \frac{\kappa}{8}\|U_j - U_j'\|^2
    \right).
\end{equation}
Given a hypothesis space of potential functions $\Hpot$ (to be defined later), the training procedure for the Brenier GAN consists of solving the optimization problem:
\begin{equation}
    \inf_{\phi \in \Hpot} \, \sup_{D \in \Hdis} \hat{\mathcal{L}}_{n, \kappa, \gamma}(\phi, D).
\end{equation}

\subsection{Hypothesis spaces}

We construct a hypothesis space of potential functions that includes both the Brenier potential and the potential described in Proposition~\ref{prop:main_gen}, but without enforcing the convexity constraint on all functions within the space. 

Throughout the section, when Assumption \ref{assumption:regularity} is satisfied, we denote $\phi^* \in C^{3,\alpha}([0,1]^d,\R)$ the Brenier potential such that $\mu^*=(\nabla \phi^*)_{\#}\lambda$. Following Remark \ref{rmk:strong_convexity}, let $M\in (1,\infty)$ such that $\frac{1}{M}\Id \le (\Hess  \phi^*)(x)\le M \Id$ for all $x\in [0,1]^d$. For simplicity, we adopt the notation $\beta:= \frac{1}{M}$.

\begin{assumption}\label{assump:potential}
Suppose Assumption \ref{assumption:regularity} is satisfied. 

  We denote for all $\epsilon>0$ by $\mathcal A^\phi(\epsilon)$ the architecture of the neural network $\phi_\epsilon$ with ReCU activation function constructed in Proposition \ref{prop:main_gen} and by $H=\|\phi_\epsilon\|_{C^{2,1}([0,1]^d,\R)}$ its $C^{2,1}([0,1]^d, \R)$-norm. We know that $\frac{1}{2M}\Id \le(\Hess  \phi_\epsilon)(x)\le 2M \Id$ for all $x\in [0,1]^d$.
  
  We assume that for all $\epsilon>0$ there exist
  $\tilde H\ge H$ such that
  the hypothesis space $\Hpot(\epsilon)$
  of generators consists of all networks $\phi\colon [0,1]^d\to \R$ with ReCU activation functions and neural network architecture $\mathcal A^\phi(\epsilon)$ such that $\|\phi\|_{C^{2,1}([0,1]^d,\R)}\le \tilde H$.

  Additionally, for any $\kappa \in (0, +\infty)$, we denote $\Hpot_{\kappa}(\varepsilon)$ the restriction of $\Hpot(\varepsilon)$ to functions $\phi$ that are $\kappa$-strongly convex.
\end{assumption}

\begin{remark}\label{rem:lip}~
\begin{itemize}
    \item $\Hpot(\varepsilon) $ is a bounded subset of $ C^{2,1}([0,1]^d, \R)$. In particular, for all $\phi \in \Hpot(\varepsilon)$, $\Hess \,\phi $ is $\tilde{H}$-Lipschitz.
   \item The hypothesis space $\Hpot_{\beta/2}(\varepsilon)$ of potentials exactly corresponds (up to gradient) to the hypothesis space of generators defined in Section \ref{sec:error_theory}, Assumption \ref{assump:generator} (recall that $\beta=1/M$).
   \end{itemize}
\end{remark}
The definition of the hypothesis space for the discriminator remains unchanged:

\begin{assumption}\label{assump:disc_2}
Suppose that Assumption \ref{assump:potential} is satisfied. 
For every $\epsilon>0$ let $\mathcal A^D(\epsilon)$ and $D_{\min}(\epsilon)$ denote the network architecture and the cut-off from Proposition \ref{prop:main_disc}, respectively.
We assume that for all $\epsilon>0$ the hypothesis space of discriminators $\Hdis(\epsilon)$ is a relatively compact (with respect to the supremum norm on $[0,1]^d$) subset of the functions $D\colon \R^d \to [0,1]$ such that $D|_{\R^d\setminus[0,1]^d}=0$, $D|_{[0,1]^d}$ is a neural network with ReLU activation function and network architecture $\mathcal A^D(\epsilon)$ and it holds that $D|_{[0,1]^d}\colon [0,1]^d\to [D_{\min}(\epsilon),1-D_{\min}(\epsilon)]$. Moreover, we assume that for every $\phi\in\Hpot_{\beta/2}(\epsilon)$ the function $D(\nabla \phi)$ constructed in Proposition \ref{prop:main_disc} is contained in 
$\Hdis(\epsilon)$ (i.e.\ $\{D(\nabla \phi),\,\phi\in\Hpot_{\beta/2}(\epsilon)\}\subseteq \Hdis(\epsilon)$).
\end{assumption}

\subsection{Strong convexity guarantee}

\begin{propo}\label{prop:penalty_hard_constraint} Suppose Assumption \ref{assump:potential} is satisfied and let $\epsilon >0$.
For every $\eta \in (0,\frac{\beta}{2})$ there exists $\zeta>0$ depending on $\varepsilon$, $\beta$ and $\tilde{H}$, 
independent of $n$, such that for every $\phi \in \Hpot(\varepsilon)$ that is not $\left(\frac{\beta}{2} - \eta\right)$-strongly convex, it holds
\begin{equation}
P^{\left( \frac{\beta}{2}\right)}(\phi) \geq \zeta.
\end{equation}
\end{propo}

\begin{proof}

Let $\eta \in (0,\frac{\beta}{2})$ and let $\phi \in \Hpot(\varepsilon)$ be not $\left(\frac{\beta}{2} - \eta\right)$-strongly convex. Then there exists $x_0 \in [0, 1]^d$ and $v_0 \in \mathbb S^1=\{v\in \R^d:\|v\|=1\}$ such that 
\begin{equation}\label{eq:spectral_radius_control}
    \langle v_0, \Hess\, \phi (x_{0}) v_0 \rangle \leq \frac{\beta}{2} - \eta.
\end{equation}
Next, note that for all $x,y\in [0,1]^d$, $v,w\in \mathbb S^1$ we have
\begin{align}
   |\langle v, \Hess\, \phi (x) v \rangle-\langle w, \Hess\, \phi (y), w \rangle |
   &\le 
   |\langle v, (\Hess\, \phi (x)-\Hess\, \phi (y)) v \rangle | +
   |\langle v, \Hess\, \phi (y) v \rangle -\langle w, \Hess\, \phi (y) w \rangle | \nonumber\\
   &=
   |\langle v, (\Hess\, \phi (x)-\Hess\, \phi (y)) v \rangle | +
   |\langle v+w, \Hess\, \phi (y) (v-w) \rangle |\nonumber \\
   &\le \tilde{H}\|x-y\|+2\tilde{H}\|v-w\| \nonumber\\
   &\le 3\tilde{H}\max\{\|x-y\|,\|v-w\|\}.
\end{align}
This implies for all $x\in B_{\eta/(6\tilde{H})}(x_0)\cap [0,1]^d$, $v\in B_{\eta/(6\tilde{H})}(v_0)\cap \mathbb S^1$,
\begin{equation}
   \langle v, \Hess\, \phi (x) v \rangle =
   \langle v, \Hess\, \phi (x) v \rangle -\langle v_0, \Hess\, \phi (x_0) v_0 \rangle +\langle v_0, \Hess\, \phi (x_0) v_0 \rangle 
   \le 
   \frac{\eta}{2} + \frac{\beta}{2} - \eta = \frac{\beta-\eta}2.
\end{equation}
Next let $B=\{(x,y)\in ([0,1]^d)^2: x,y \in B_{\eta/(6\tilde{H})}(x_0), x\neq y, \frac{x-y}{\|x-y\|}\in B_{\eta/(6\tilde{H})}(v_0)\}$. Then 
we have by Taylor's formula, for all $(x,y)\in B$, that
\begin{equation}
\begin{split}
    &\phi\left(\frac{x+y}{2}\right) - \frac{1}{2} ( \phi(x) +\phi(y))
    =
    -\frac{1}{2}\int_0^1 \left(\frac{1}{2}-\left|t-\frac{1}{2}\right|\right)\langle x-y,(\Hess \phi) (x+t(y-x))
    (x-y)\rangle dt\\
    &\ge 
    -\frac{1}{2}\int_0^1 \left(\frac{1}{2}-\left|t-\frac{1}{2}\right|\right)
    \frac{\beta-\eta}{2}\|x-y\|^2 dt
    = -
    \frac{\beta-\eta}{16}\|x-y\|^2
\end{split}
\end{equation}
and consequently
\begin{equation}
    \phi\left(\frac{x+y}{2}\right) - \frac{1}{2} ( \phi(x) +\phi(y))
    + \frac{\beta}{16}\|x-y\|^2
    \ge 
    \frac{\eta}{16}\|x-y\|^2.
\end{equation}
From this we conclude that
\begin{equation}
    P^{(\frac{\beta}{2})}(\phi)
\ge 
 \int_{B} \mathrm{ReLU}\left( \phi\left(\frac{x+y}{2}\right) - \frac{1}{2} ( \phi(x) +\phi(y))+ \frac{\beta}{16} \Vert x - y \Vert^2  \right) d(x,y)\ge \zeta, 
\end{equation}
where $\zeta=\frac{\eta}{16}\int_B \|x-y\|^2d(x,y)$. Note that $\zeta>0$ because $B \neq \emptyset$ and that $\zeta$ does not depend on $\phi$ (but only on $\beta$, $\eta$ and $\tilde{H}$). This completes the proof.
\end{proof}

\begin{propo}\label{prop:ULLN_result_penalty}
Suppose Assumption \ref{assump:potential} is satisfied.
  It holds a.s. that
    \begin{equation}\label{eq:ULLN_result_penalty}
        \sup_{\phi \in \Hpot(\varepsilon)} \, | \hat{P}^{(\beta)}_n(\phi) - P^{(\beta)}(\phi) | \underset{n \rightarrow +\infty}{\longrightarrow} 0.
    \end{equation}
\end{propo}

\begin{proof}
    The goal is to apply the uniform law of large numbers as stated in \cite[Theorem 16]{ferguson2017course}. To this end, we first note that $\Hpot(\varepsilon)$ is relatively compact in $C^0([0,1]^d, \R)$ because it is bounded in $C^{2,1}([0,1]^d, \R)$ and the compact inclusion $C^{2,1}([0,1]^d, \R) \subseteq C^0([0,1]^d, \R)$ holds. 
    We define for all $(u, u')\in [0,1]^d$, for all $\phi \in \Hpot(\varepsilon)$, 
    \begin{equation}
        h((u,u'), \phi) = \mathrm{ReLU}\left( \phi\left(\frac{u+u'}{2}\right) - \frac{1}{2} ( \phi(u) +\phi(u'))+ \frac{\beta}{8} \Vert u - u' \Vert^2  \right).
    \end{equation}
    For all $(u, u')\in [0,1]^d$, $\phi \mapsto h((u,u'), \phi)$ is continuous. In addition, the compactness of $[0,1]^d$ implies that there exists a constant $c>0$ such that for all $(u, u')\in [0,1]^d$, for all $\phi \in \Hpot(\varepsilon)$, $|h((u,u'), \phi)| \leq c$. The right-hand term has a finite expected value with respect to the probability measure $\mathbb{P}_U$.
    Thus, all the conditions to apply the uniform law of large numbers are met and imply \eqref{eq:ULLN_result_penalty}.
\end{proof}


We conclude this section with the following proposition. It states that for a penalty parameter sufficiently large and enough samples, the potential estimated when solving the min-max optimization problem is ensured to be strongly convex.
\begin{propo}\label{prop:penalty_hard_constraint_empirical}
Suppose Assumption \ref{assump:disc_2} is satisfied and let $\epsilon>0$. Let $\eta \in (0,\frac{\beta}{2})$.
Then there exists $\gamma >0$ and $\bar{n} \in \N$, such that for all $n \geq \bar{n}$, 
\begin{equation}
 \inf_{\phi \in \Hpot(\varepsilon)}\sup_{D \in \Hdis(\varepsilon)} \hat{L}_n(\nabla \phi, D) + \gamma \hat{P}_n^{\left( \beta/2\right)}(\phi) = \inf_{\phi \in \Hpot_{\frac{\beta}{2}-\eta}(\varepsilon)}\sup_{D \in \Hdis(\varepsilon)} \hat{L}_n(\nabla \phi, D) + \gamma \hat{P}_n^{\left( \beta/2\right)}(\phi).
\end{equation}
\end{propo}

\begin{proof}
 We start noting that, according to the definition of the hypothesis space of discriminators, there exist constants $c_-, c_+ \in \R$ with $c_+>c_-$ such that for every $\phi \in \Hpot(\varepsilon)$ and $n\in \N$
\begin{equation}
c_- \leq \sup_{D \in \Hdis(\varepsilon)} \hat{L}_n(\nabla \phi, D) \leq c_+.
\end{equation}
Additionally, we have $\phi_\varepsilon \in \Hpot_{\beta/2}(\varepsilon) \subseteq \Hpot(\varepsilon)$. Thus, $P^{\left( \frac{\beta}{2}\right)}(\phi_\varepsilon) =0$ and 
\begin{equation}
\inf_{\phi \in  \Hpot(\varepsilon)} \sup_{D \in \Hdis(\varepsilon)} \hat{L}_n(\nabla \phi, D) + \gamma \hat{P}_n^{\left( \beta/2\right)}(\phi)  \leq c_+.
\end{equation}
According to Proposition \ref{prop:penalty_hard_constraint}, there exists $\zeta>0$ so that for every $\phi \in \Hpot(\varepsilon)$ that is not $\left(\frac{\beta}{2} - \eta\right)$-strongly convex, $P^{\left( \frac{\beta}{2}\right)}(\phi) \ge  \zeta$. 

Additionally, according to Proposition \ref{prop:ULLN_result_penalty}, there exists $\bar{n} \in \N$ such that for all $n \geq \bar{n}$ and $\phi \in \Hpot(\varepsilon)$  we have $|\hat{P}_n^{\left( \frac{\beta}{2}\right)}(\phi) - P^{\left( \frac{\beta}{2}\right)}(\phi) | \leq \zeta/2$. We deduce that $\hat{P}_n^{\left( \frac{\beta}{2}\right)}(\phi) \geq \zeta/2$. 

Finally, let $\gamma>\frac{2(c_+-c_-)}{\zeta}$. For every
 $\phi \in \Hpot(\varepsilon)$ that is not $\left(\frac{\beta}{2} - \eta\right)$-strongly convex we have
\begin{align}
    \sup_{D \in \Hdis(\varepsilon)} \hat{L}_n(\nabla \phi, D) + \gamma \hat{P}_n^{\left( \beta/2\right)}(\phi) &\ge c_-+\gamma \hat{P}_n^{\left( \frac{\beta}{2}\right)}(\phi) \nonumber \\
    &\ge c_-+\gamma \zeta/2 \nonumber \\
    &>c_+\nonumber \\
    &\ge \inf_{\bar{\phi} \in  \Hpot(\varepsilon)} \sup_{D \in \Hdis(\varepsilon)} \hat{L}_n(\nabla \phi, D) + \gamma \hat{P}_n^{\left( \beta/2\right)}(\bar{\phi}).
\end{align}
    This concludes the proof.
\end{proof}

\section{Experiments}

We conduct numerical experiments to evaluate the Brenier GAN on several datasets: (i) 2D Gaussian mixtures, used as a sanity check, and (ii) three grayscale image datasets — MNIST, Fashion-MNIST, and (iii) NORB. The results demonstrate that the Brenier-GAN effectively approximates complex target distributions while simultaneously learning a convex potential, consistent with the theoretical formulation.

\subsection{Synthetic 2D example}
\begin{figure}
    \centering
    \begin{tabular}{cccc}
       \rotatebox[origin=l]{90}{True distribution} & \includegraphics[width=0.3\linewidth]{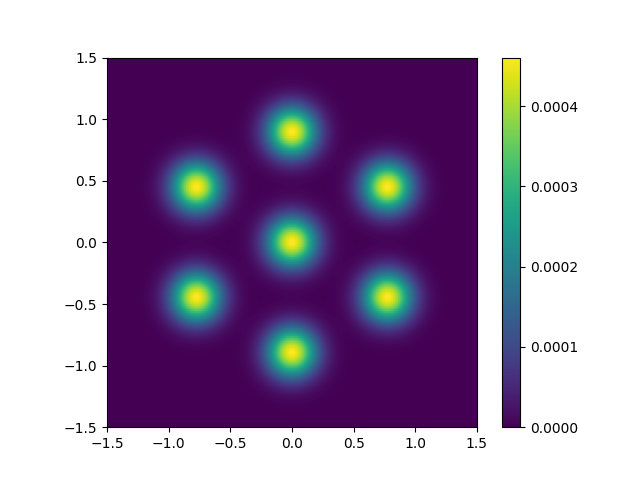} &  \\
       \rotatebox[origin=l]{90}{\qquad $\gamma=0.0$}  & \includegraphics[width=0.3\linewidth]{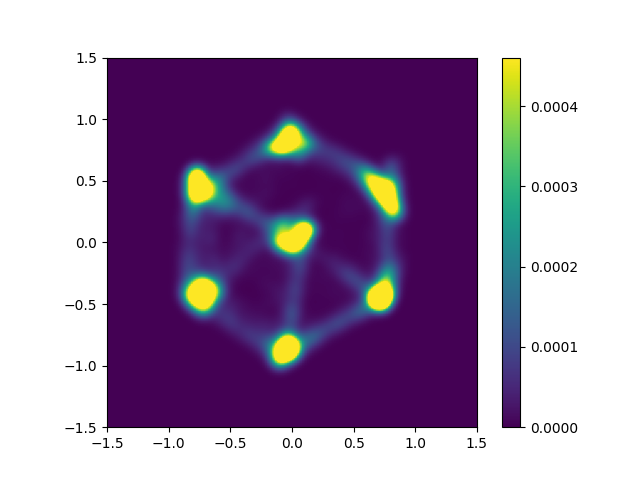} & \includegraphics[width=0.3\linewidth]{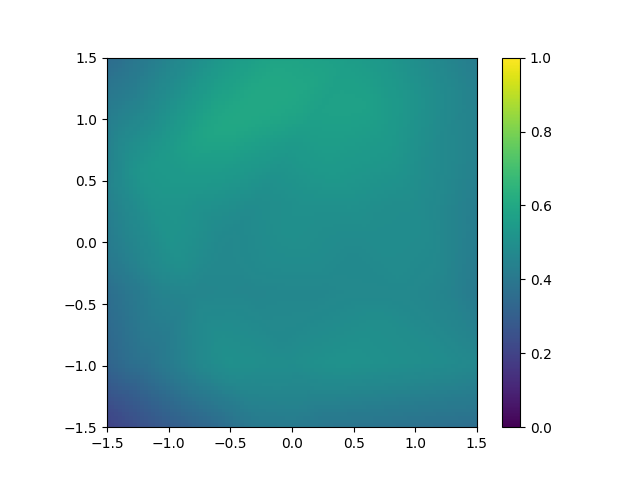} & \includegraphics[width=0.35\linewidth]{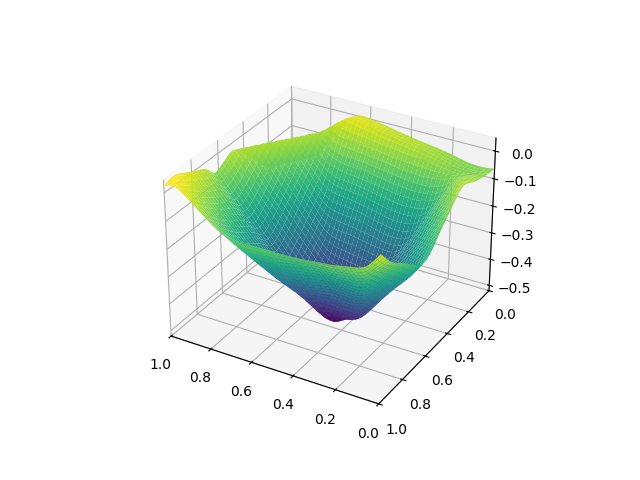} \\
        \rotatebox[origin=l]{90}{\qquad $\gamma=0.03$}  & \includegraphics[width=0.3\linewidth]{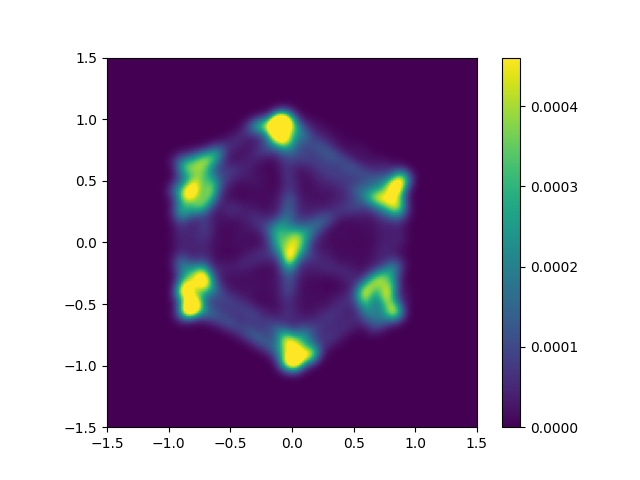} & \includegraphics[width=0.3\linewidth]{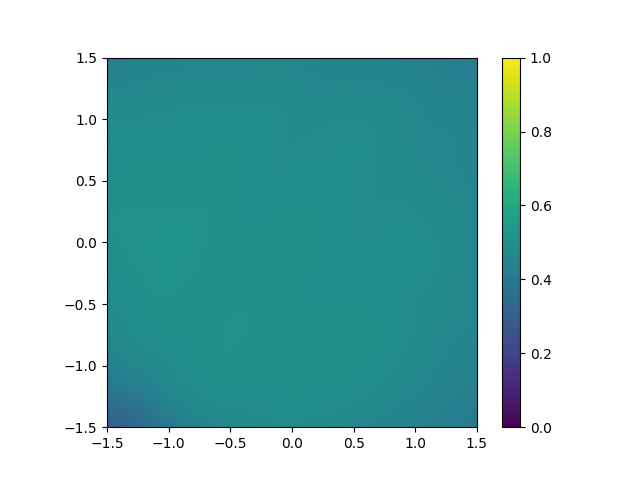} & \includegraphics[width=0.35\linewidth]{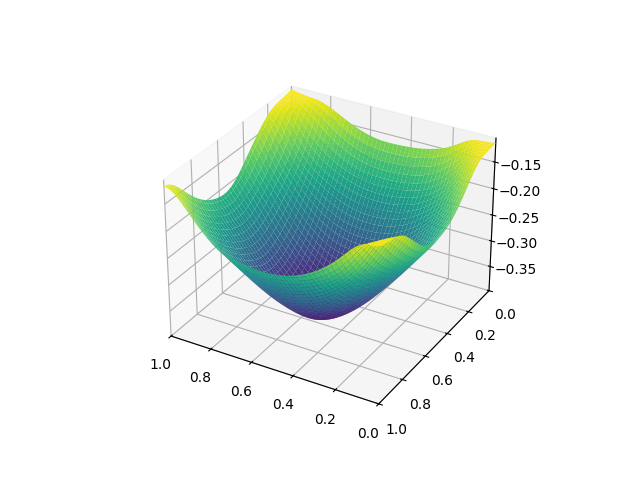} \\
        \rotatebox[origin=l]{90}{\qquad $\gamma=0.06$}  & \includegraphics[width=0.3\linewidth]{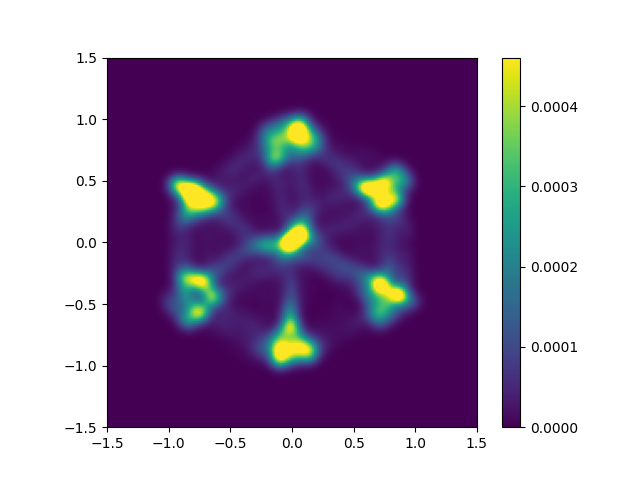} & \includegraphics[width=0.3\linewidth]{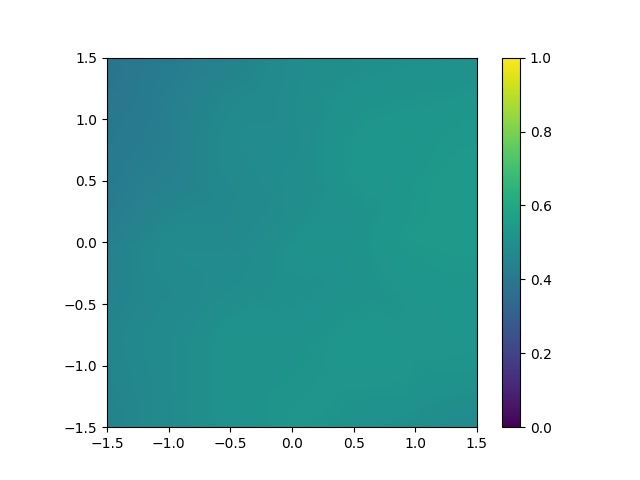} & \includegraphics[width=0.35\linewidth]{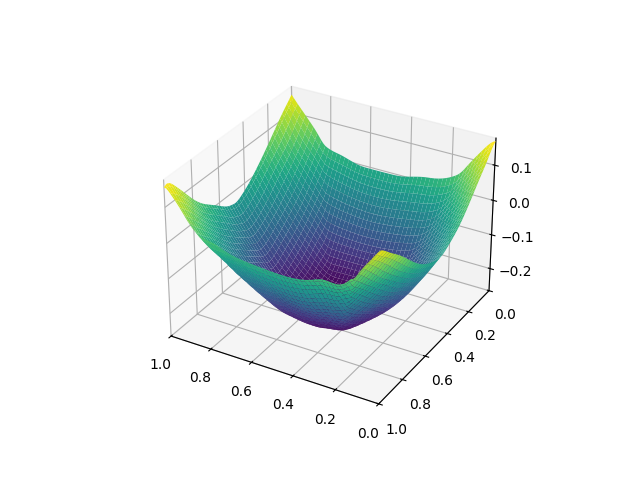} \\
        \rotatebox[origin=l]{90}{\qquad $\gamma=0.1$}  & \includegraphics[width=0.3\linewidth]{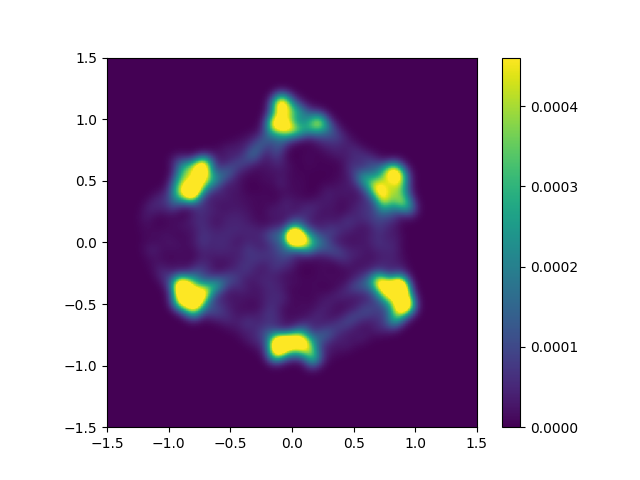} & \includegraphics[width=0.3\linewidth]{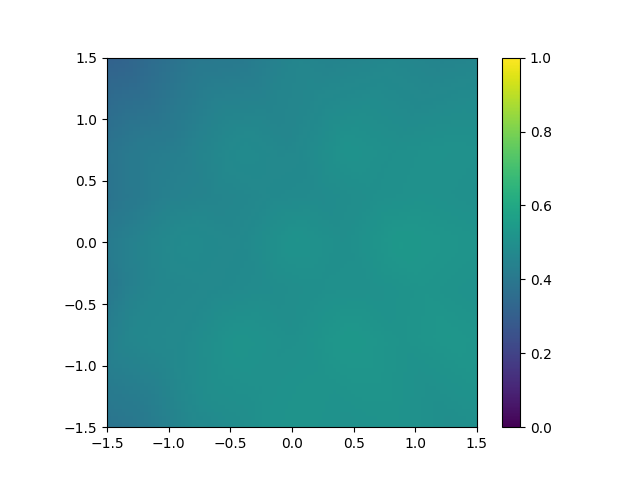} & \includegraphics[width=0.35\linewidth]{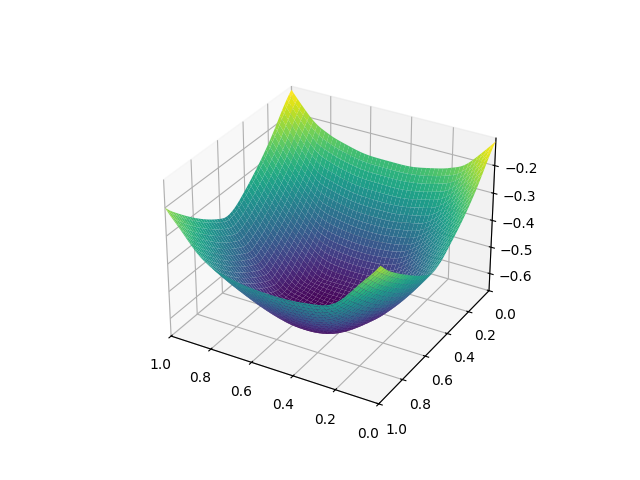} 
    \end{tabular}
    \caption{Output of the Brenier GAN for different values of the penalty parameter $\gamma$ after 1000 epochs: (left) estimated distribution; (center) learned discriminator; (right) learned potential function.}
    \label{fig:results_GMM}
\end{figure}

We first implement and evaluate the Brenier GAN on a synthetic example. Here the target distribution, denoted as $\mu^*$, is a 2D Gaussian Mixture Model (GMM) with seven modes, as shown in Figure \ref{fig:results_GMM} (first row). Each mode has a covariance matrix $\sigma^2 I_d$ with $\sigma^2 = 0.02$. We generate a training dataset with $60\,000$ samples.

For training, we set the learning rate to $0.001$, batch-size to 1000 and the strong convexity parameter $\kappa = 0.1$. The generator network (i.e., the network modeling the Brenier potential) is a 5-layer fully connected architecture with ReCU activations and hidden layers of width 16, 32, 64, 32, and a final output layer in $\R$. The discriminator network includes three fully connected layers with widths 128, 64, and 1, using LeakyReLU activations and a Sigmoid output. To mitigate mode collapse, we reinitialized the discriminator’s weights every 50 epochs until epoch 500. Due to computational complexity we evaluate the convexity loss based on $20$ samples.

The Brenier GAN was trained for 1000 epochs. The estimated density after training, visualized in Figure \ref{fig:results_GMM} (right) using a kernel density estimator with bandwidth 0.1, closely approximates the true distribution $\mu^*$. 

For this experiment, it is worth mentioning that we observed sensitivity to the parameter $\kappa$, where lower values led to accurate mode detection but underestimated the variance, distorting the mode shapes, while too high values of $\kappa$ tend to connect more the modes between them. We selected the value of $\kappa$ achieving the best tradeoff. 

The results obtained for different values of the penalty parameter $\gamma$ are shown in Figure~\ref{fig:results_GMM}. We observe that the choice $\gamma = 0.1$ appears suitable for enforcing the convexity of the potential, which is not convex when $\gamma \in \{0.0, 0.03, 0.06\}$. Moreover, the use of this penalization term has little to no impact on the estimation of the target distribution. As expected, the discriminator becomes approximately constant at 0.5 by the end of training, indicating that it can no longer distinguish between real and generated samples.

\subsection{Handwritten digits and fashion generation}
To showcase the applicability of the Brenier-GAN to computable real world problems, we trained Brenier-GANs to generated handwritten digits based on the MNIST dataset~\cite{lecun1998mnist} as well as T-shirts based on the Fashion-MNIST dataset~\cite{xiao2017fashionmnist}.
The MNIST dataset consists of $60\,000$ gray-scale trainings images of handwritten digits from $0$ to $9$. Each digit is centered in a $32\times32$ image. The Fashion-MNIST dataset adapts the format of MNIST but increases the complexity by moving to $10$ different articles of clothing and accessory.
Example images from both datasets are depicted in \cref{fig:mist-fashionmnist}. 
\begin{figure}
    \centering
    \begin{subfigure}{0.8\linewidth}
        \includegraphics[width=\linewidth]{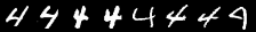}
        \caption{Original MNIST data}
    \end{subfigure}
    \begin{subfigure}{0.8\linewidth}
        \includegraphics[width=\linewidth]{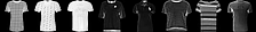}
        \caption{Original Fashion-MNIST data}
    \end{subfigure}
    \caption{Examples of the original datasets.}
    \label{fig:mist-fashionmnist}
\end{figure}
\paragraph{Implementation details}
For generating images, we feed vectors $x\in[-1,1]^d$ where $d = 1042$ corresponds to the flattened image dimension to a $5$-layer fully connected architecture implementing the Brenier potential. To meet the requirements of an increased capacity, we increase the width of the hidden layers to $1050$, $1096$, $512$, $128$ and $1$. As discriminator serves a $3$-layer fully connected architecture with an initial width of $512$ neurons which is halved for each deeper layer. We train with a linear learning rate decay with an initial learning rate of $0.00005$ for $2000$ epochs. Due to computational complexity we evaluate the convexity loss based on $10$ samples.
We reinitialize the weights of the discriminator when its loss drops below $0.001$ or exceeds $50$ to ensure meaningful gradient steps for the generator.
The convexity parameters are set to $\gamma=1$ and $\kappa=0.000001$ for MNIST and $\kappa=0.0001$ for Fashion-MNIST if not stated otherwise.
\paragraph{Results}
Experiments have shown that sampling uniformly from a zero centered bounded set improves the convergence of the Brenier-GAN. As a consequence, we use $x \in [-1,1]^d$ instead of sampling uniformly from the unit cube $[0,1]^d$.
This shift does not interfere with the theoretical results but improves the generation process in practice.
Generated images, both for a handwritten digit and a t-shirt, are shown in \cref{fig:mnist_fashionmnist_inference}. \Cref{fig:mnist-gamma=1,fig:fashion-gamma=1} show the results for a Brenier-GAN trained \textit{with} convexity constraints and \cref{fig:mnist-gamma=0,fig:fashion-gamma=0} show examples for a training \textit{without} convexity constraints.
\begin{figure}
    \centering
    \begin{subfigure}{0.8\linewidth}
        \includegraphics[width=\linewidth]{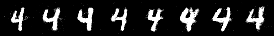}
        \caption{Brenier-GAN generated images of the handwritten digit 4. The Brenier-GAN was trained with $\gamma=1$ and $\kappa=0.000001$ for $1980$ epochs.}
        \label{fig:mnist-gamma=1}
    \end{subfigure}
    \begin{subfigure}{0.8\linewidth}
        \includegraphics[width=\linewidth]{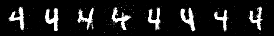}
        \caption{Brenier-GAN generated images of the handwritten digit 4 after roughly $1980$ Epochs. The Brenier-GAN was trained without convexity constraints ($\gamma=\kappa=0$) for $1980$ Epochs.}
        \label{fig:mnist-gamma=0}
    \end{subfigure}
        \begin{subfigure}{0.8\linewidth}
        \includegraphics[width=\linewidth]{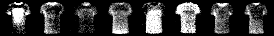}
        \caption{Brenier-GAN generated images of the class `t-shirt/top'. The Brenier-GAN was trained on the Fashion-MNIST t-shirt class for $1980$ epochs with $\gamma=1$ and $\kappa=0.0001$.}
        \label{fig:fashion-gamma=1}
    \end{subfigure}
        \begin{subfigure}{0.8\linewidth}
        \includegraphics[width=\linewidth]{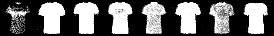}
        \caption{Brenier-GAN generated images of the class `t-shirt/top'. The Brenier-GAN was trained on the Fashion-MNIST t-shirt class for $1980$ epochs without convexity constraints ($\gamma=\kappa=0$).}
        \label{fig:fashion-gamma=0}
    \end{subfigure}
    \caption{Visual examples of Brenier-GAN generated images.}
    \label{fig:mnist_fashionmnist_inference}
\end{figure}
The results show that qualitatively the generated images approximate the true distribution. In contrast to the synthetic $2$D experiments, our experiments suggest that smaller $\kappa$ values lead to more realistic results.
Our Fashion-MNIST experiments show that the convexity constraints mitigate mode collapse since for the unconstrained Brenier-GAN often a white t-shirt is generated. Visual inspecting the generation results of both generators leads to the conclusion that the Brenier-GAN is able to approximate the true distribution with and without the convexity constraints. However, the diversity increases when applying convexity constraints during training by the cost of a bit more noisy images.  
The generation quality is exemplary compared in \cref{fig:fashion-gamma=1,fig:fashion-gamma=0}. In \cref{fig:mnist_convexity,fig:fashionmnist_convexity}, we show the convexity term evaluated on $1000$ randomly drawn noise input vectors during the training process for a Brenier-GAN trained with $\gamma=1, \kappa=0.000001$ and $\gamma=1,\kappa=0.0001$ respectively (green). The convexity term evaluated for a Brenier-GAN trained without penalty term ($\gamma=\kappa=0$) is depicted in blue. The plots show that enforcing convexity on $10$ randomly sampled points with our convexity loss already allows to control the convexity of the overall Brenier potential.
\begin{figure}[t]
    \centering
\begin{subfigure}{0.87\linewidth}
    \centering
    \includegraphics[width=\linewidth]{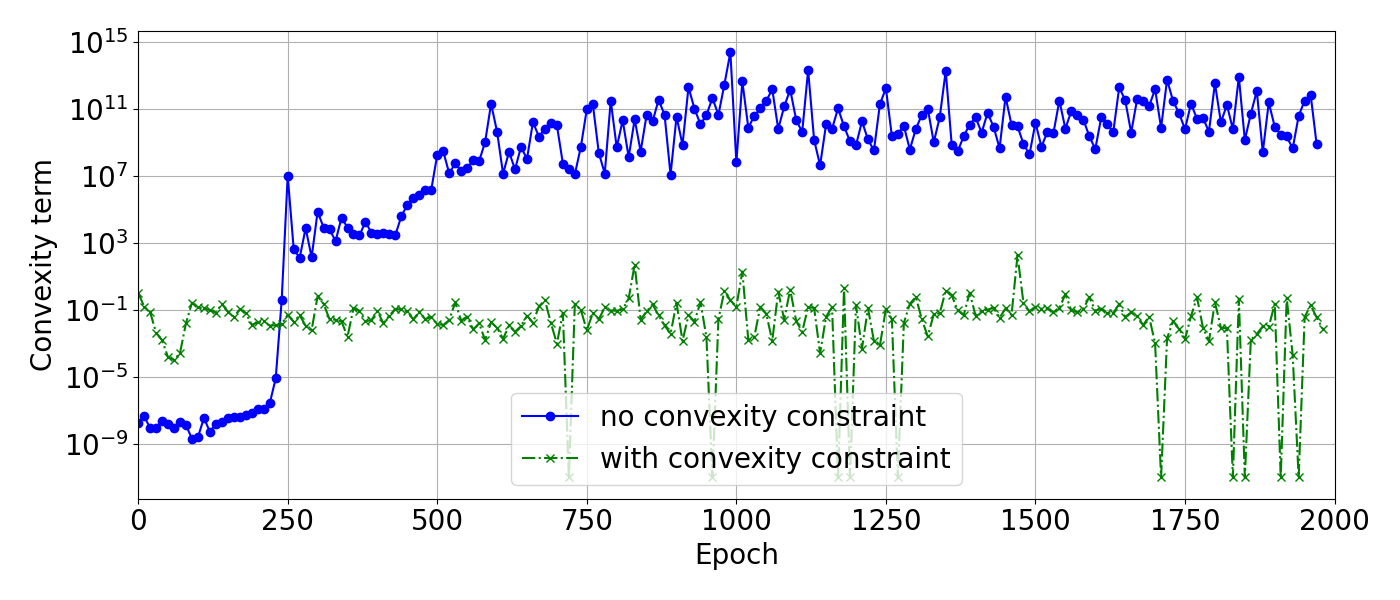}
    \caption{MNIST}
    \label{fig:mnist_convexity}
\end{subfigure}
\begin{subfigure}{0.87\linewidth}
    \centering
    \includegraphics[width=\linewidth]{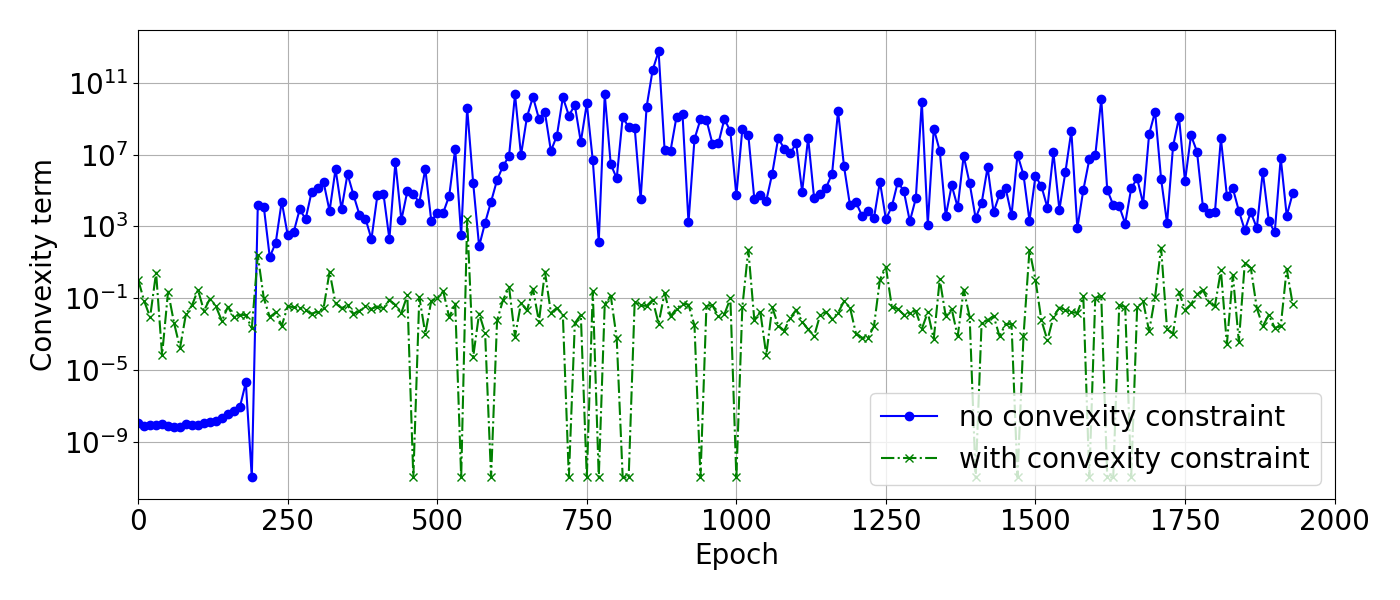}
    \caption{Fashion-MNIST}
    \label{fig:fashionmnist_convexity}
\end{subfigure}
    \caption{Convexity of the learned Brenier potential with and without constraints during training. To allow log-scale visualization, convexity terms which equal $0$ are replaced by a small epsilon of $10^{-10}$.}
    \label{fig:convexity}
\end{figure}

\subsection{NORB}
With the last experiments on the NORB dataset \cite{lecun2004learning}, we can demonstrate that the BrenierGAN is also able to generate images whose objects are displayed in three dimensions. 
The dataset was originally created for 3D object recognition from shape and its training data is given by five generic classes presented by toys, each containing five different instances. 
The images show the 3D objects without colour and texture features against a uniform background. The variation in the data comes from only six different lighting conditions, nine elevations and 18 azimuths set during the recordings by two different cameras (see \cref{fig:norb}). The difference between the images from the different camera positions is barely noticeable to the eye, so we only used the left camera to avoid overfitting and longer GAN training time with redundant data. 
This results in $4860$ images per class, with each image having a resolution of $96 \times 96$ pixels.
As with the MNIST and FashionMNIST experiments, we will only focus on one class and present the results for the truck class.
To sum up, the dataset is challenging in three ways: We have a small amount of data, a small variation in the data, and somehow a depth component, since the objects shown are three dimensional and have a shadow.
\paragraph{Implementation details}
We adopt the same network architecture and parameter settings as used for MNIST and FashionMNIST, with the following exceptions: The training is performed for $4000$ epochs with a batch size of $256$, and the best-performing parameters found through trial and error are set to $\gamma = 0.001$ and $\kappa = 0.00001$. Moreover, the generator employs the ReQU activation function (instead on ReCU) between its layers to improve stability during the GAN training process. To assess the effect of the convexity regularization, we also evaluate experiments with $\gamma = \kappa = 0$. For computational tractability, all input images are resized to a resolution of $32 \times 32$ pixels. In line with the (Fashion)MNIST experiments, the generator received a uniformly distributed noise vector consisting of $1024$ elements within the range $[-1,1]$ as input. While MNIST and FashionMNIST are widely used benchmarks for GAN training, the NORB dataset is not commonly employed in this context. Therefore, we provide a baseline given by the original GAN \cite{vanilla_gan} with a generator composed of five fully connected layers with 128, 256, 512 and $1,024$ neurons in the hidden layers, and a discriminator composed of three fully connected layers with hidden layers of 512 and 256 widths.
\paragraph{Results}
\Cref{fig:norbs_inference} demonstrates that VanillaGAN has limited capability in generating clear object structures, with the synthesized trucks appearing more as accumulations of scattered points rather than coherent forms. In contrast, BrenierGAN significantly outperforms the VanillaGAN baseline. 
Experiments conducted both with and without the convexity parameter indicate that BrenierGAN is capable of generating trucks with diverse shapes, even when employing $\kappa$ and $\gamma$ values that are smaller than those used in the (Fashion)MNIST experiments. 
Notably, training without the convexity term results in trucks with clearer and more distinct shapes, although a tendency to generate them at a certain angle is observed. Conversely, incorporating the convexity term enables the generator to capture a broader range of angles, albeit at the expense of finer structural details. 
As in the (Fashion)MNIST experiments, the generated images exhibit a degree of noise, which could potentially be mitigated through appropriate postprocessing techniques in Python. 
Finally, the convexity loss plot (see \cref{fig:norb_convexity}) supports the findings from the (Fashion)MNIST experiments, demonstrating that even with these very small $\gamma$ and $\kappa$ values, enforcing convexity effectively controls the overall Brenier potential.
\begin{figure}
    \centering
        \includegraphics[width=0.8\linewidth]{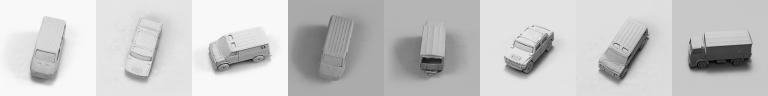}
        \caption{Original NORB data for the class ``truck''.}
        \label{fig:norb}
\end{figure}
\begin{figure}
    \centering
    \begin{subfigure}{0.8\linewidth}
        \includegraphics[width=\linewidth]{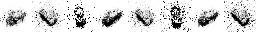}
        \caption{Samples generated by a Vanilla GAN trained for $4000$ epochs representing a baseline.}
        \label{fig:norb_vgan}
    \end{subfigure}
        \begin{subfigure}{0.8\linewidth}
        \includegraphics[width=\linewidth]{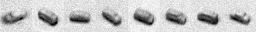}
        \caption{Samples generated by the Brenier GAN trained  on the convexity constraints $\gamma= 0.001$ and $\kappa= 0.00001$ for $4000$ Epochs.}
        \label{fig:norb_convex_on}
    \end{subfigure}
        \begin{subfigure}{0.8\linewidth}
        \includegraphics[width=\linewidth]{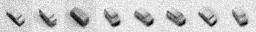}
        \caption{Samples generated by the Brenier GAN trained  without convexity constraints ($\gamma=\kappa=0$) for $4000$ Epochs.}
        \label{fig:norb_convex_off}
    \end{subfigure}
    \caption{Visual examples of GAN generated images for the NORB dataset.}
    \label{fig:norbs_inference}
\end{figure}
\begin{figure}[!ht]
    \centering
    \includegraphics[width=0.87\linewidth]{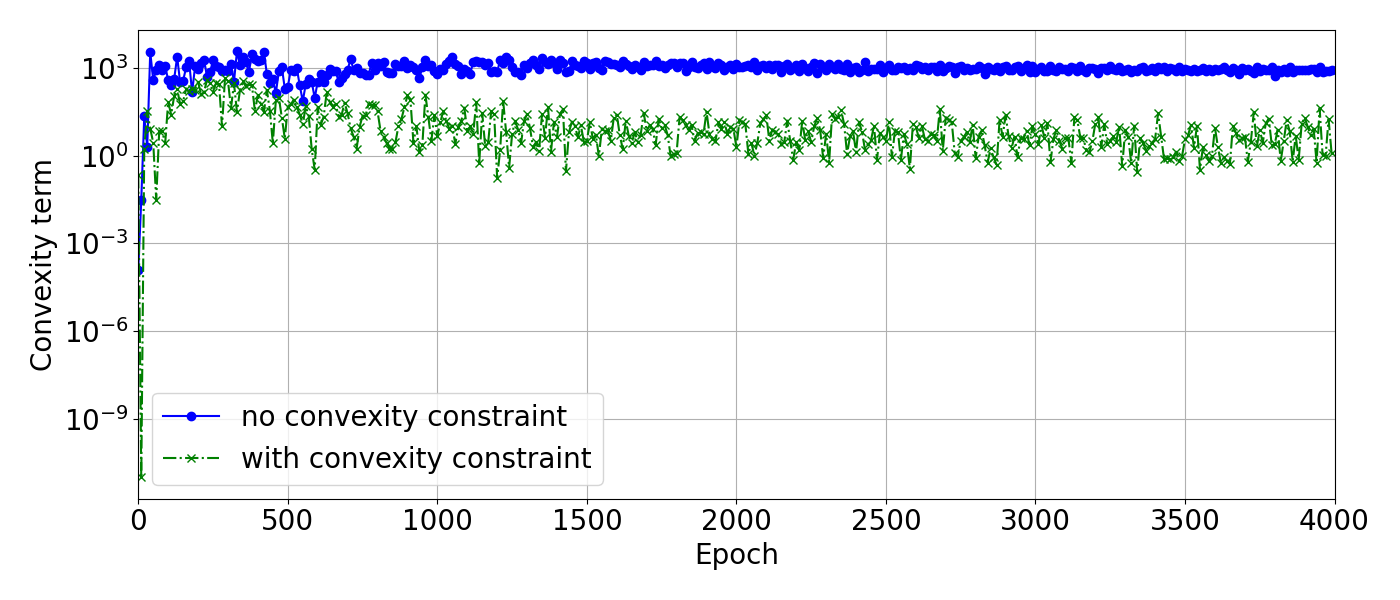}
    \caption{Convexity of the learned Brenier potential with and without constraints during training.}
    \label{fig:norb_convexity}
\end{figure}

\section{Conclusion}
In this work, we introduced the Brenier GAN, a generative model grounded in optimal transport theory. By modeling the generator as the gradient of a convex potential, our framework brings structure into the GAN setting and enables a rigorous statistical learning analysis. Leveraging recent results on universal approximation with ReCU neural networks, we showed that the ground truth target distribution can be accurately approximated, and that the learning error, measured by Jensen–Shannon divergence, can be made arbitrarily small as the sample size increases.

To enforce convexity of the potential numerically, we proposed a simple yet effective regularization technique based on penalizing non-convex behavior at random samples. Our theoretical results are supported by numerical experiments on low-dimensional synthetic data and small grayscale images, which suggest that both a convex potential and the target distribution are learned accurately in practice.

Future work could extend the statistical learning analysis presented in this paper to Brenier-GANs using convolutional neural network architectures. This extension would likely improve training efficiency and scalability, particularly for high-dimensional image data.

\section*{Acknowledgments}
The authors thank Patrick Krüger and Gabriele Steidl for interesting discussions. Large parts of the computations were carried out on the PLEIADES cluster at the University of Wuppertal, which was supported by the Deutsche Forschungsgemeinschaft (DFG, grant No. INST 218/78-1 FUGG) and the Bundesministerium für Bildung und Forschung (BMBF). S\'egol\`ene Martin's work is funded by the Deutsche Forschungsgemeinschaft (DFG, German Research Foundation) under Germany's Excellence Strategy – The Berlin Mathematics
Research Center MATH+ (EXC-2046/1, project ID: 390685689).
Annika M\"utze acknowledges support through the junior research group project ``UnrEAL'' by the German Federal Ministry of Education and Research (BMBF), grant no.\ 01IS22069.

\bibliographystyle{ieeetr}
\bibliography{bibliography}

\begin{thebibliography}{10}

\bibitem{bond2021deep}
S.~Bond-Taylor, A.~Leach, Y.~Long, and C.~G. Willcocks, ``Deep generative
  modelling: A comparative review of {VAE}s, {GAN}s, normalizing flows,
  energy-based and autoregressive models,'' {\em IEEE Transactions on Pattern
  Analysis and Machine Intelligence}, vol.~44, no.~11, pp.~7327--7347, 2021.

\bibitem{harshvardhan2020comprehensive}
G.~Harshvardhan, M.~K. Gourisaria, M.~Pandey, and S.~S. Rautaray, ``A
  comprehensive survey and analysis of generative models in machine learning,''
  {\em Computer Science Review}, vol.~38, p.~100285, 2020.

\bibitem{montufar2011refinements}
G.~Montufar and N.~Ay, ``Refinements of universal approximation results for
  deep belief networks and restricted boltzmann machines,'' {\em Neural
  Computation}, vol.~23, no.~5, pp.~1306--1319, 2011.

\bibitem{salakhutdinov2009deep}
R.~Salakhutdinov and G.~Hinton, ``Deep {B}oltzmann machines,'' in {\em
  Artificial Intelligence and Statistics Conference}, pp.~448--455, PMLR, 2009.

\bibitem{du2019implicit}
Y.~Du and I.~Mordatch, ``Implicit generation and modeling with energy based
  models,'' {\em Advances in Neural Information Processing Systems}, vol.~32,
  2019.

\bibitem{dinh2022density}
L.~Dinh, J.~Sohl-Dickstein, and S.~Bengio, ``Density estimation using {R}eal
  {NVP},'' in {\em International Conference on Learning Representations}, 2022.

\bibitem{papamakarios2021normalizing}
G.~Papamakarios, E.~Nalisnick, D.~J. Rezende, S.~Mohamed, and
  B.~Lakshminarayanan, ``Normalizing flows for probabilistic modeling and
  inference,'' {\em Journal of Machine Learning Research}, vol.~22, no.~57,
  pp.~1--64, 2021.

\bibitem{teshima2020coupling}
T.~Teshima, I.~Ishikawa, K.~Tojo, K.~Oono, M.~Ikeda, and M.~Sugiyama,
  ``Coupling-based invertible neural networks are universal diffeomorphism
  approximators,'' {\em Advances in Neural Information Processing Systems},
  vol.~33, pp.~3362--3373, 2020.

\bibitem{chan2023lu}
R.~Chan, S.~Penquitt, and H.~Gottschalk, ``{LU}-{N}et: Invertible neural
  networks based on matrix factorization,'' in {\em International Joint
  Conference on Neural Networks}, pp.~1--10, IEEE, 2023.

\bibitem{vanilla_gan}
I.~Goodfellow, J.~Pouget-Abadie, M.~Mirza, B.~Xu, D.~Warde-Farley, S.~Ozair,
  A.~Courville, and Y.~Bengio, ``Generative adversarial networks,'' {\em
  Advances in Neural Information Processing Systems}, vol.~3, 06 2014.

\bibitem{nowozin2016f}
S.~Nowozin, B.~Cseke, and R.~Tomioka, ``f-gan: Training generative neural
  samplers using variational divergence minimization,'' {\em Advances in Neural
  Information Processing Systems}, vol.~29, 2016.

\bibitem{munoz2021temporal}
A.~Munoz, M.~Zolfaghari, M.~Argus, and T.~Brox, ``Temporal shift {GAN} for
  large scale video generation,'' in {\em IEEE/CVF Winter Conference on
  Applications of Computer Vision}, pp.~3179--3188, 2021.

\bibitem{wang2018esrgan}
X.~Wang, K.~Yu, S.~Wu, J.~Gu, Y.~Liu, C.~Dong, Y.~Qiao, and C.~Change~Loy,
  ``E{SRGAN}: Enhanced super-resolution generative adversarial networks,'' in
  {\em Proceedings of the IEEE/CVF European Conference on Computer Vision
  workshops}, 2018.

\bibitem{chen2018neural}
R.~T. Chen, Y.~Rubanova, J.~Bettencourt, and D.~K. Duvenaud, ``Neural ordinary
  differential equations,'' {\em Advances in Neural Information Processing
  Systems}, vol.~31, 2018.

\bibitem{lipman2022flow}
Y.~Lipman, R.~T. Chen, H.~Ben-Hamu, M.~Nickel, and M.~Le, ``Flow matching for
  generative modeling,'' in {\em International Conference on Learning
  Representations}, 2023.

\bibitem{liu2022flow}
X.~Liu, C.~Gong, and Q.~Liu, ``Flow straight and fast: Learning to generate and
  transfer data with rectified flow,'' in {\em International Conference on
  Learning Representations}, 2023.

\bibitem{sohl2015deep}
J.~Sohl-Dickstein, E.~Weiss, N.~Maheswaranathan, and S.~Ganguli, ``Deep
  unsupervised learning using nonequilibrium thermodynamics,'' in {\em
  International Conference on Machine Learning}, pp.~2256--2265, PMLR, 2015.

\bibitem{welling2011bayesian}
M.~Welling and Y.~W. Teh, ``Bayesian learning via stochastic gradient
  {L}angevin dynamics,'' in {\em International Conference on Machine Learning},
  pp.~681--688, Citeseer, 2011.

\bibitem{rombach2022high}
R.~Rombach, A.~Blattmann, D.~Lorenz, P.~Esser, and B.~Ommer, ``High-resolution
  image synthesis with latent diffusion models,'' in {\em IEEE/CVF Conference
  on Computer Vision and Pattern Recognition}, pp.~10684--10695, 2022.

\bibitem{brenier1991polar}
Y.~Brenier, ``Polar factorization and monotone rearrangement of vector-valued
  functions,'' {\em Communications on Pure and Applied Mathematics}, vol.~44,
  no.~4, pp.~375--417, 1991.

\bibitem{paty2020regularity}
F.-P. Paty, A.~d’Aspremont, and M.~Cuturi, ``Regularity as regularization:
  Smooth and strongly convex {B}renier potentials in optimal transport,'' in
  {\em International Conference on Artificial Intelligence and Statistics},
  pp.~1222--1232, PMLR, 2020.

\bibitem{santambrogio2015optimal}
F.~Santambrogio, {\em Optimal Transport for Applied Mathematicians}, vol.~55.
\newblock Springer, 2015.

\bibitem{asatryan2020convenient}
H.~Asatryan, H.~Gottschalk, M.~Lippert, and M.~Rottmann, ``A convenient
  infinite dimensional framework for generative adversarial learning,'' {\em
  arXiv preprint arXiv:2011.12087}, 2020.

\bibitem{biau2020some}
G.~Biau, B.~Cadre, M.~Sangnier, and U.~Tanielian, ``Some theoretical properties
  of {GAN}s,'' {\em The Annals of Statistics}, vol.~48, no.~3, pp.~1539--1566,
  2020.

\bibitem{puchkin2024rates}
N.~Puchkin, S.~Samsonov, D.~Belomestny, E.~Moulines, and A.~Naumov, ``Rates of
  convergence for density estimation with generative adversarial networks,''
  {\em Journal of Machine Learning Research}, vol.~25, no.~29, pp.~1--47, 2024.

\bibitem{belomestny2023simultaneous}
D.~Belomestny, A.~Naumov, N.~Puchkin, and S.~Samsonov, ``Simultaneous
  approximation of a smooth function and its derivatives by deep neural
  networks with piecewise-polynomial activations,'' {\em Neural Networks},
  vol.~161, pp.~242--253, 2023.

\bibitem{yarotsky2017error}
D.~Yarotsky, ``Error bounds for approximations with deep relu networks,'' {\em
  Neural Networks}, vol.~94, pp.~103--114, 2017.

\bibitem{amos2017input}
B.~Amos, L.~Xu, and J.~Z. Kolter, ``Input convex neural networks,'' in {\em
  International Conference on Machine Learning}, pp.~146--155, PMLR, 2017.

\bibitem{huangconvex}
C.-W. Huang, R.~T. Chen, C.~Tsirigotis, and A.~Courville, ``Convex potential
  flows: Universal probability distributions with optimal transport and convex
  optimization,'' in {\em International Conference on Learning
  Representations}, 2021.

\bibitem{korotin2019wasserstein}
A.~Korotin, V.~Egiazarian, A.~Asadulaev, A.~Safin, and E.~Burnaev,
  ``Wasserstein$\-$2 generative networks,'' in {\em International Conference on
  Machine Learning}, 2021.

\bibitem{chen2018optimal}
Y.~Chen, Y.~Shi, and B.~Zhang, ``Optimal control via neural networks: A convex
  approach,'' in {\em International Conference on Learning Representations},
  2019.

\bibitem{richter2021input}
J.~Richter-Powell, J.~Lorraine, and B.~Amos, ``Input convex gradient
  networks,'' {\em arXiv preprint arXiv:2111.12187}, 2021.

\bibitem{chaudhari2024gradient}
S.~Chaudhari, S.~Pranav, and J.~M. Moura, ``Gradient networks,'' {\em arXiv
  preprint arXiv:2404.07361}, 2024.

\bibitem{belomestny2021rates}
D.~Belomestny, E.~Moulines, A.~Naumov, N.~Puchkin, and S.~Samsonov, ``Rates of
  convergence for density estimation with generative adversarial networks,''
  {\em arXiv:2102.00199}, 2023.

\bibitem{schumaker2007spline}
L.~Schumaker, {\em Spline {F}unctions: {B}asic {T}heory}.
\newblock Cambridge University Press, 2007.

\bibitem{villani2008optimal}
C.~Villani {\em et~al.}, {\em Optimal {T}ransport: {O}ld and {N}ew}, vol.~338.
\newblock Springer, 2008.

\bibitem{caffarelli1996boundary}
L.~A. Caffarelli, ``Boundary regularity of maps with convex potentials --
  {II},'' {\em Annals of Mathematics}, vol.~144, no.~3, pp.~453--496, 1996.

\bibitem{weyl1939volume}
H.~Weyl, ``On the volume of tubes,'' {\em American Journal of Mathematics},
  vol.~61, no.~2, pp.~461--472, 1939.

\bibitem{gray2003tubes}
A.~Gray, {\em Tubes}, vol.~221.
\newblock Springer Science \& Business Media, 2003.

\bibitem{ferguson2017course}
T.~S. Ferguson, {\em A course in large sample theory}.
\newblock Routledge, 2017.

\bibitem{lecun1998mnist}
Y.~LeCun, ``The {MNIST} database of handwritten digits,'' {\em
  http://yann.lecun.com/exdb/mnist/}, 1998.

\bibitem{xiao2017fashionmnist}
H.~Xiao, K.~Rasul, and R.~Vollgraf, ``Fashion-{MNIST}: a novel image dataset
  for benchmarking machine learning algorithms,'' {\em arXiv:1708.07747}, 2017.

\bibitem{lecun2004learning}
Y.~LeCun, F.~J. Huang, and L.~Bottou, ``Learning methods for generic object
  recognition with invariance to pose and lighting,'' in {\em IEEE/CVF
  Conference on Computer Vision and Pattern Recognition}, vol.~2, pp.~II--104,
  IEEE, 2004.

\end{thebibliography}

\end{document}